\newif\ifdraft
\drafttrue
\draftfalse

\newif\ifarxiv
\arxivtrue
\newcommand{\iflong}[2]{\ifarxiv{}#1\else{}#2\fi}

\documentclass{article}
\usepackage{ijcai23}

\usepackage{times}
\usepackage{soul}
\usepackage{url}
\usepackage[hidelinks]{hyperref}
\usepackage[utf8]{inputenc}
\usepackage[small]{caption}
\usepackage{graphicx}
\usepackage{amsmath}
\usepackage{amsthm}
\usepackage{bbold}
\usepackage{booktabs}
\usepackage{algorithm}
\usepackage{algorithmic}
\usepackage[switch]{lineno}
\usepackage{tikz}
\usepackage{siunitx}

\usepackage{amsmath,amssymb}
\usepackage{ifthen}

\usepackage{bm}

\usepackage{subcaption}
\usepackage{listings}

\usepackage{letltxmacro}
\newcommand*{\SavedLstInline}{}
\LetLtxMacro\SavedLstInline\lstinline
\DeclareRobustCommand*{\lstinline}{
  \ifmmode
    \let\SavedBGroup\bgroup
    \def\bgroup{
      \let\bgroup\SavedBGroup
      \hbox\bgroup
    }
  \fi
  \SavedLstInline
}
\newcommand{\code}[1]{\texttt{\lstinline[mathescape,classoffset=1,keywordstyle=\color{black},basicstyle=\color{black},classoffset=0,keywordstyle=\color{black}]{#1}}}
\newcommand{\lstcommentcolor}[1]{\textcolor{darkorange}{#1}}

\usepackage[capitalize]{cleveref}
\crefname{listing}{Algorithm}{Algorithms}
\Crefname{listing}{Algorithm}{Algorithms}

\definecolor{darkred}{rgb}{.5,0,0}
\definecolor{darkgreen}{rgb}{0,.5,0}
\definecolor{darkblue}{rgb}{0,0,.5}
\definecolor{darkorange}{rgb}{.8,.4,0}
\ifdraft
\newcommand{\levi}[1]{\textcolor{darkgreen}{\emph{(LL: #1)}}}
\newcommand{\mh}[1]{\textcolor{darkblue}{\emph{(MH: #1)}}}
\newcommand{\todo}[1]{\textcolor{darkorange}{(\emph{TODO: #1})}}
\newcommand{\comment}[1]{\textcolor{gray}{(\emph{#1})}}
\newcommand{\warning}[1]{\textcolor{red}{(\emph{WARNING: #1})}}
\newcommand{\quest}[1]{\textcolor{darkgreen}{(\emph{Q: #1})}}

\else
\newcommand{\levi}[1]{}
\newcommand{\mh}[1]{}
\newcommand{\todo}[1]{}
\newcommand{\comment}[1]{}
\newcommand{\warning}[1]{}
\newcommand{\quest}[1]{}

\fi

\newcounter{alphoversetcount}

\newcommand{\resetalph}{\setcounter{alphoversetcount}{0}}

\newenvironment{alphalign*}{
\csname align*\endcsname\resetalph{}
}{
\csname endalign*\endcsname\resetalph{}
}

\newcommand{\ie}{\emph{i.e.}, }
\newcommand{\eg}{\emph{e.g.}, }

\DeclareMathOperator*{\argmin}{argmin}

\newcommand{\Reals}{{\mathbb R}}

\newcommand{\indicator}[1]{\left[\!\left[#1\right]\!\right]}

\newcommand{\eps}{\varepsilon}

\newcommand{\dop}{\tfrac{d}{\pol}}

\newcommand{\prodmix}{p_{\times}}

\newcommand{\epsmix}{\eps_{\text{mix}}}
\newcommand{\epslow}{\eps_{\text{low}}}

\newcommand{\Loss}{L}
\newcommand{\loss}{\ell}

\newcommand{\regret}{\mathcal{R}}

\newcommand{\actionset}{\mathcal{A}}
\newcommand{\contextset}{\mathcal{Q}}

\newcommand{\betaset}{\mathcal{B}}

\newcommand{\nodeset}{\mathcal{N}}
\newcommand{\leafset}{\mathcal{L}}
\newcommand{\children}{\mathcal{C}}

\newcommand{\nodesetcost}{\overline{\nodeset}}

\newcommand{\rootop}{\operatorname{root}}
\newcommand{\parent}{\text{par}}
\newcommand{\anc}{\text{anc}}
\newcommand{\ancn}{\text{anc}_+}
\newcommand{\desc}{\text{desc}}
\newcommand{\descn}{\text{desc}_+}
\newcommand{\pol}{\pi}

\newcommand{\mutexset}{mutex set}

\newcommand{\mutexsets}{mutex sets}

\newcommand{\patternset}{\mathcal{M}}

\newtheorem{theorem}{Theorem}
\crefname{theorem}{Theorem}{Theorems}
\newtheorem{lemma}[theorem]{Lemma}
\crefname{lemma}{Lemma}{Lemmata}

\crefname{corollary}{Corollary}{Corollaries}
\newtheorem{example}[theorem]{Example}
\crefname{example}{Example}{Examples}
\newtheorem{remark}[theorem]{Remark}
\crefname{remark}{Remark}{Remarks}

\newcommand{\citet}[1]{{\citeauthor{#1}~\shortcite{#1}}}

\title{Levin Tree Search with Context Models}
 
\author{
Laurent Orseau$^1$
\and
Marcus Hutter$^1$\and
Levi H. S. Lelis$^{2}$
\affiliations
$^1$Google DeepMind\\
$^2$Department of Computing Science, University of Alberta, Canada\\
and Alberta Machine Intelligence Institute (Amii), Canada\\
\emails
\{lorseau,mhutter\}@google.com,
levi.lelis@ualberta.ca
}

\begin{document}

\maketitle

\begin{abstract}
Levin Tree Search (LTS) is a search algorithm that makes use of a policy (a probability distribution over actions) 
and comes with a theoretical guarantee on the number of expansions before reaching a goal node, depending on the quality of the policy. 
This guarantee can be used as a loss function, which we call the LTS loss, to optimize neural networks representing the policy (LTS+NN). 
In this work we show that the neural network can be substituted with parameterized context models originating from the online compression literature (LTS+CM). 
We show that the LTS loss is convex under this new model,
which allows for using standard convex optimization tools,
and obtain convergence guarantees to the optimal parameters in an online setting for a given set of solution trajectories --- guarantees that cannot be provided for neural networks. 
The new LTS+CM algorithm compares favorably against LTS+NN on several benchmarks: Sokoban (Boxoban), The Witness, and the 24-Sliding Tile puzzle (STP). The difference is particularly large on STP, where LTS+NN fails to solve most of the test instances while LTS+CM solves each test instance in a fraction of a second.
Furthermore, we show that LTS+CM is able to learn a policy that solves the Rubik's cube in only a few hundred expansions, which considerably improves upon previous machine learning techniques.
\end{abstract}

\section{Introduction}\label{sec:Intro}

We
\footnote{
\ifarxiv Extended version of the IJCAI 2023 paper. 
Source code at: \url{https://github.com/google-deepmind/levintreesearch_cm}.
\else
Appendices: \url{http://arxiv.org/abs/2305.16945}. Source code: \url{https://github.com/deepmind/levintreesearch_cm}.
\fi}
consider the problem of solving a set of deterministic single-agent search problems of a given domain,
by starting with little prior domain-specific knowledge.
We focus on algorithms that learn from previously solved instances
to help solve the remaining ones.
We consider the satisficing setting where solvers should (learn to) quickly find a solution, rather than to minimize the cost of the returned solutions.

Levin Tree Search (LevinTS, LTS) is a tree search algorithm for this setup that uses a policy, \ie a probability distribution over actions, to guide the search~\cite{orseau2018single}.
LTS has a guarantee on the number of search steps required before finding a solution, which
depends on the probability of the corresponding sequence of actions as assigned by the policy.
\citet{orseau2021policy} showed that this guarantee can be used as a loss function. 
This LTS loss is used to optimize a neural-network (NN) policy in the context of the Bootstrap search-and-learn process~\cite{ArfaeeZH11}:
The NN policy is used in LTS (LTS+NN) to iteratively solve an increasing number of problems from a given set, optimizing the parameters of the NN when new problems are solved to improve the policy by minimizing the LTS loss. 

One constant outstanding issue with NNs is that the loss function (whether quadratic, log loss, LTS loss, etc.) is almost never convex in the NN's parameters.
Still, most of the time NNs are trained using online convex optimization algorithms,
such as stochastic gradient descent, Adagrad~\cite{duchi2011adaptive}, and its descendants.
Such algorithms often come with strong convergence or regret guarantees that only hold under convexity assumptions, and can help to understand the effect of various quantities (number of parameters, etc.) on the learning speed~\cite{Zin03,hazan2016oco,boyd2004convex}. 
In this paper we present parameterized context models for policies that are convex with respect to the model's parameters for the LTS loss. Such models guarantee that we obtain an optimal policy in terms of LTS loss for a given set of training trajectories
--- a guarantee NNs do not have. 

The context models we introduce for learning policies are based on the models from the online data compression literature~\cite{rissanen1983universal,willems1995ctw}. 
Our context models are composed of a set of contexts, where each context is associated with a probability distribution over actions.
These distributions are combined using product-of-experts~\cite{hinton2002poe} to produce the policy used during the LTS search.
The expressive power of product-of-experts comes mainly from the ability of each expert to (possibly softly) veto a particular option by assigning it a low probability. 
A similar combination using geometric mixing~\cite{mattern2013geomix,matthew2005adaptive} (a geometrically-parameterized variant of product-of-experts) in a multi-layer architecture has already proved competitive with NNs in classification, regression and density modelling tasks~\cite{veness2017gln,veness2021gln,budden2020gated}.
In our work the context distributions are fully parameterized and we show that the LTS loss is convex for this parameterization.

In their experiments, \citet{orseau2021policy} showed that LTS+NN performs well on two of the three evaluated domains (Sokoban and The Witness), but fails to learn a policy for the 24-Sliding Tile Puzzle (STP).
We show that LTS with context models optimized with the LTS loss within the Bootstrap process is able to learn a strong policy for all three domains evaluated, including the STP. We also show that LTS using context models is able to learn a policy that allows it to find solutions to random instances of the Rubik's Cube with only a few hundred expansions. In the context of satisficing planning, this is a major improvement over previous machine-learning-based approaches, which require hundreds of thousands expansions to solve instances of the Rubik's Cube. 

We start with giving some notation and the problem definition (\Cref{sec:Not}), 
before describing the LTS algorithm, for which 
we also provide a new lower bound on the number of node expansions (\Cref{sec:LTS}).
Then, we describe parameterized context models and explain why we can expect them to work well when using product-of-experts (\Cref{sec:CM}),
before showing that the LTS loss function is convex for this parameterization (\Cref{sec:Conv}) and considering theoretical implications.
Finally we present the experimental results (\Cref{sec:Exp}) before concluding (\Cref{sec:Conc}).

\section{Notation and Problem Definition}\label{sec:Not}

A table of notation can be found in \iflong{\cref{apdx:table_notation}}{Appendix I}.
We write $[t] = \{1,2,\dots t\}$ for a natural number $t$.
The set of nodes is $\nodeset$ and is a forest, where each tree in the forest represents a search problem with the root being the initial configuration of the problem. 
The set of children of a node $n\in\nodeset$ is $\children(n)$ and its parent is $\parent(n)$;
if a node has no parent it is a root node.
The set of ancestors of a node is $\anc(n)$ and is the transitive closure of $\parent(\cdot)$;
we also define $\ancn(n)=\anc(n)\cup\{n\}$.
Similarly, $\desc(n)$ is the set of the descendants of $n$, and $\descn(n)=\desc(n)\cup\{n\}$.
The depth of a node is $d(n)= |\anc(n)|$, and so the depth of a root node is 0.
The root $\rootop(n)$ of a node $n$ is the single node $n_0\in\ancn(n)$ such that $n_0$ is a root.
A set of nodes $\nodeset'$ is a tree in the forest $\nodeset$ if and only if there is a node $n^0\in\nodeset'$ such that  $\bigcup_{n\in\nodeset'}\rootop(n)=\{n^0\}$.
Let $\nodeset^0 = \bigcup_{n\in\nodeset} \rootop(n)$ be the set of all root nodes.
We write $n_{[j]}$ for the node at depth $j\in[d(n)]$ on the path from $\rootop(n)=n_{[0]}$ to $n=n_{[d(n)]}$.
Let $\nodeset^*\subseteq\nodeset$ be the set of all \emph{solution} nodes,
and we write $\nodeset^*(n)=\nodeset^*\cap\descn(n)$ for the set of solution nodes
under $n$.
A \emph{policy} $\pol$ is such that for all $n \in \nodeset$ and for all $n'\in\children(n): \pol(n' \mid n) \geq 0$
and $\sum_{n'\in\children(n)} \pol(n'\mid n) \leq 1$.
The policy is called \emph{proper} if the latter holds as an equality.
We define, for all $n'\in\children(n)$, $\pol(n') = \pol(n)\pol(n'\mid n)$ recursively and 
$\pol(n)=1$ if $n$ is a root node.

Edges between nodes are labeled with \emph{actions} and the children of any node all have different labels, but different nodes can have overlapping sets of actions.
The set of all edge labels is $\actionset$.
\warning{check that $A$ (now removed) does not appear in the text}
Let $a(n)$ be the label of the edge from $\parent(n)$ to $n$,
and let $\actionset(n)$ be the set of edge labels for the edges from node $n$ to its children.
Then $n\neq n' \land \parent(n)=\parent(n')$ implies
$a(n)\neq a(n')$.
\warning{We also abuse notation and use $a$ for an edge label too}

Starting at a given root node $n^0$, a tree search algorithm 
expands a set $\nodeset'\subseteq\descn(n^0)$ until it finds a solution node
in $\nodeset^*(n^0)$.
In this paper, given a set of root nodes,
we are interested in parameterized algorithms that attempt to minimize the cumulative number of nodes
that are expanded before finding a solution node for each root node,
by improving the parameters of the algorithm from found solutions,
and with only little prior domain-specific knowledge.

\section{Levin Tree Search}\label{sec:LTS}

Levin Tree Search (LevinTS, which we abbreviate to LTS here) is a tree/graph search algorithm based on best-first search~\cite{pearl1984heuristics} that uses the cost function
\footnote{\citet{orseau2018single} actually use the cost function $(d(n)+1)/\pol(n)$.
Here we use $d(n)/\pol(n)$ instead which is actually (very) slightly better and makes the notation simpler.
All original results can be straightforwardly adapted.}
$n\mapsto d(n)/\pol(n)$~\cite{orseau2018single},
which, for convenience, we abbreviate as $\dop(n)$.
That is, since $\dop(\cdot)$ is monotonically increasing from parent to child,
LTS expands all nodes by increasing order of $\dop(\cdot)$ (Theorem 2, \citet{orseau2018single}).

\begin{theorem}[LTS upper bound, adapted from \citet{orseau2018single}, Theorem 3]\label{thm:upper_bound}
Let $\pol$ be a policy.
For any node $n^*\in\nodeset$,
let $\nodesetcost(n^*) = \{n\in\nodeset: \rootop(n)=\rootop(n^*)\land \dop(n) \leq \dop(n^*)\}$ be
the set of nodes within the same tree
with cost at most that of $n^*$.
Then
\begin{align*}
    |\nodesetcost(n^*)| \leq 1+\frac{d(n^*)}{\pol(n^*)}\,.
\end{align*}
\end{theorem}
\begin{proof}
Let $\leafset$ be the set of leaves of $\nodesetcost(n^*)$, then
\begin{align*}
    |\nodesetcost(n^*)| &\leq 1+ \sum_{n\in\leafset} d(n) = 1+ \sum_{n\in\leafset} \pol(n)\dop(n) \\
    &\leq 1+ \sum_{n\in\leafset} \pol(n)\dop(n^*) 
    \leq 1+\dop(n^*)\,,
\end{align*}
where we used \iflong{\cref{lem:sum_to_one}}{Lemma 10} (in Appendix) on the last inequality.
\end{proof}

The consequence is that LTS started at $\rootop(n^*)$ expands at most $1+\dop(n^*)$ nodes before reaching $n^*$.

\citet{orseau2021policy} also provides a related lower bound showing that, for any policy, there are sets of problems
where any algorithm needs to expand $\Omega(\dop(n^*))$ nodes before reaching some node $n^*$ in the worst case.
They also turn the guarantee of \cref{thm:upper_bound} into a loss function, used to optimize the parameters of a neural network.
Let $\nodeset'$ be a set of solution nodes whose roots are all different, define the \emph{LTS loss function}:
\begin{align}\label{eq:loss}
    \Loss(\nodeset') = \sum_{n\in\nodeset'} \dop(n)
\end{align}
which upper bounds the total search time of LTS to reach all nodes in $\nodeset'$.
\Cref{eq:loss} is the loss function used in \iflong{\cref{alg:bootstrap} (\cref{apdx:bootstrap})}{Algorithm 2 (Appendix A)}
to optimize the policy --- but a more precise definition for context models will be given later.
To further justify the use of this loss function,
we provide a lower bound on the number of expansions that LTS must perform before reaching an (unknown) target node.

\begin{theorem}[Informal lower bound]\label{thm:informal_lower_bound}
For a proper policy $\pol$ and any node $n^*$,
the number of nodes whose $\dop$ cost is at most that of $n^*$
is at least $[\frac{1}{\bar d}\dop(n^*)-1]/(|\actionset|-1)$,
where $\bar d-1$ is the average depth of the leaves of those nodes.
\end{theorem}

A more formal theorem is given in \iflong{\cref{apdx:lower_bound}}{Appendix B}.

\begin{example}
For a binary tree with a uniform policy, 
since $\bar d=d(n^*)+1$,
the lower bound gives $2^{d}d/(d+1)-1$ nodes for a node $n^*$ at depth $d$ and of probability $2^{-d}$, which is quite tight since the tree has $2^d - 1$ nodes.
The upper bound $1+d2^d$ is slightly looser.
\end{example}

\begin{remark}
Even though pruning (such as state-equivalence pruning) can make the policy improper, 
in which case the lower bound does not hold and the upper bound can be loose,
optimizing the parameters of the policy for the upper bound still makes sense, since pruning can be seen as a feature
placed on top of the policy --- that is, the policy is optimized as if pruning is not used.
It must be noted that for optimization \citet{orseau2021policy} (Section 4) use the log gradient trick to replace the upper bound loss with the actual number of expansions in an attempt to account for pruning; as the results of this paper suggest, it is not clear whether one should account for the actual number of expansions while optimizing the model.
\end{remark}

\section{Context Models}\label{sec:CM}

Now we consider that the policy $\pol$ has some parameters $\beta\in\betaset$
(where $\betaset\subseteq\Reals^k$ for some $k$, which will be made more precise later) 
and we write $\pol(\cdot;\beta)$ when the parameters are relevant to the discussion.
As mentioned in the introduction, we want the LTS loss function of \cref{eq:loss}
to be convex in the policy's parameters,
which means that we cannot use just any policy --- in particular this rules out deep neural networks.
Instead, we use context models, which have been widely used in online prediction and compression (\eg \cite{rissanen1983universal,willems1995ctw,matthew2005adaptive,veness2021gln}).

The set of contexts is $\contextset$.
A context is either active or inactive at a given node in the tree.
At each node $n$, the set of active contexts is $\contextset(n)$,
and the policy's prediction at $n$ depends only on these active contexts.

Similarly to patterns in  pattern databases~\cite{culberson1998pattern},
we organize contexts in sets of mutually exclusive contexts, called \emph{\mutexsets{}},
and each context belongs to exactly one mutex set.
The set of \mutexsets{} is $\patternset$.
For every \mutexset{} $M \in \patternset$, 
for every node $n$, at most one context is active per mutex set. In this paper we are in  the case where \emph{exactly} one context is active per mutex set, which is what happens when searching with multiple pattern databases, where each pattern database provides a single pattern for a given node in the tree. 
When designing contexts, it is often more natural to directly design \mutexsets{}.
See \Cref{fig:product_mixing} for an example, omitting the bottom parts of (b) and (d) for now.

To each context $c\in\contextset$ we associate a \emph{predictor} $p_c:\actionset\to[0,1]$
which is a (parameterized) categorical probability distribution over edge labels that will be optimized from training data --- the learning part will be explained in \cref{sec:Opt}.

To combine the predictions of the active contexts at some node $n$, 
we take their renormalized product, 
as an instance of product-of-experts~\cite{hinton2002poe}: 
\begin{multline}\label{eq:prodmix}
    \forall a\in\actionset(n): 
    \prodmix(n, a) = 
    \frac{\prod_{c\in\contextset(n)}p_c(a)}{\sum_{a'\in\actionset(n)} \prod_{c\in\contextset(n)} p_c(a')}
\end{multline}
We refer to the operation of \cref{eq:prodmix} as \emph{product mixing},
by relation to geometric mixing~\cite{mattern2013geomix}, a closely related operation.
Then, one can use $\prodmix(n, a)$ to define the policy
 $\pi(n'|n)=\prodmix(n,a(n'))$ to be used with LTS.

The choice of this particular aggregation of the individual predictions is best  explained by the following example.

\begin{example}[Wisdom of the product-of-experts crowd]
\Cref{fig:product_mixing} (a) and (b) displays a simple maze environment where the agent is coming from the left. The only sensible action is to go Up (toward the exit),
but no single context sees the whole picture.
Instead, they see only individual cells around the agent, and one context also sees (only)
the previous action (which is Right).
The first two contexts only see empty cells to the left and top of the agent, and are uninformative (uniform probability distributions) about which action to take.
But the next three contexts, while not knowing what to do, know what \emph{not} to do.
When aggregating these predictions with product mixing, only one action remains with high probability: Up.
\end{example}

\begin{example}[Generalization and specialisation]
Another advantage of product mixing is its ability to make use of both general predictors
and specialized predictors.
Consider a \mutexset{} composed of $m$ contexts, and assume we have a total of $M$ data points (nodes on solution trajectories).
Due to the mutual exclusion nature of \mutexsets{}, these $M$ data points must be partitioned among the $m$ contexts.
Assuming for simplicity a mostly uniform partitioning, then each context receives approximately $M/m$ data points to learn from.
Consider the \mutexsets{} in \cref{fig:product_mixing} (b):
The first 4 \mutexsets{} have size 3 (each context can see a wall, an empty cell or the goal) and the last one has size 4.
These are very small sizes and thus the parameters of the contexts predictors should quickly see enough data to learn an accurate distribution.
However, while accurate, the distribution can hardly be \emph{specific}, and each predictor alone is not sufficient to obtain a nearly-deterministic policy --- though fortunately product mixing helps with that.
Now compare with the 2-cross \mutexset{} in \cref{fig:product_mixing} (d), and assume that cells outside the grid are walls. A quick calculation, assuming only one goal cell, gives that it should contain a little less than 1280 different contexts.
 Each of these contexts thus receives less data to learn from on average than the contexts in (b), but also sees more information from the environment which may lead to more specific (less entropic)
 distributions, as is the case in situation (c).
\end{example}

\begin{figure*}
    \begin{center}
    \scalebox{0.75}{\begin{tikzpicture}

\node[align=center] at (2,-1em) {(a)};
\node[align=center] at (9,-1em) {(b)};

\newcommand{\wallcell}[2]{\fill[darkgray] (#1,#2) rectangle (#1+1,#2+1)}

\begin{scope}[scale=0.5]
\fill[darkgray] (1,0) rectangle (7,6);
\fill[white] (2,1) rectangle (6,5);
\wallcell{3}{5};
\wallcell{4}{4};
\wallcell{4}{2};
\wallcell{4}{1};
\wallcell{3}{4};
\wallcell{2}{2};
\wallcell{1}{2};

\foreach \x in {1,2,3,4,5,6} {
  \foreach \y in {0,1,2,3,4,5} {
    \draw[thick] (\x,\y) rectangle (\x+1,\y+1);
  }
}

\def\xplay{3}
\def\yplay{1}
\draw[red, thick, fill=red] (\xplay+.1,\yplay+.1) -- (\xplay+.5,\yplay+.9) -- (\xplay+.9,\yplay+.1) -- cycle;
\draw[->,blue, very thick]  (\xplay-.2,\yplay+.5) -- (\xplay+.5,\yplay+.5);

\def\xapple{5}
\def\yapple{3}
\draw[solid, fill=darkgreen] (\xapple+.5,\yapple+.5) circle (.4);

\end{scope}

\def\xsep{1.3}

\begin{scope}[xshift=-.2cm]
    \begin{scope}[xshift=5.1cm]
        \begin{scope}[yscale=0.5,align=left, text width=2cm]
            \node[] at (0,3.5) {Left};
            \node[] at (0,2.5) {Up};
            \node[] at (0,1.5) {Right};
            \node[] at (0,0.5) {Down};
        \end{scope}
    \end{scope}
    
    \begin{scope}[xshift=5cm]
        \begin{scope}[yshift=2cm,scale=0.5]
            \draw[gray,thick] (1,1) rectangle (2,2);
            \draw[pink, thick, fill=pink] (1.1,1.1) -- (1.5,1.9) -- (1.9,1.1) -- cycle;
            \draw[black, thick] (0,1) -- (1,1) -- (1,2) -- (0,2) -- cycle;
        \end{scope}
    
        \begin{scope}[yscale=0.5]
            \foreach \y in {0,1,2,3} {
                \draw[black, thin] (0,\y) rectangle (1,\y+1);
            }
            \draw[draw opacity=0, fill=blue] (0,3) rectangle (0.25,4);
            \draw[draw opacity=0, fill=blue] (0,2) rectangle (0.25,3);
            \draw[draw opacity=0, fill=blue] (0,1) rectangle (0.25,2);
            \draw[draw opacity=0, fill=blue] (0,0) rectangle (0.25,1);
        \end{scope}
    \end{scope}

    \begin{scope}[xshift=5cm+\xsep*1cm]
        \begin{scope}[xshift=-.5cm,yshift=2cm,scale=0.5]
            \draw[gray,thick] (1,1) rectangle (2,2);
            \draw[pink, thick, fill=pink] (1.1,1.1) -- (1.5,1.9) -- (1.9,1.1) -- cycle;
            \draw[black, thick] (1,2) -- (1,3) -- (2,3) -- (2,2) -- cycle;
        \end{scope}
    
        \begin{scope}[yscale=0.5]
            \foreach \y in {0,1,2,3} {
                \draw[black, thin] (0,\y) rectangle (1,\y+1);
            }
            \draw[draw opacity=0, fill=blue] (0,3) rectangle (0.25,4);
            \draw[draw opacity=0, fill=blue] (0,2) rectangle (0.25,3);
            \draw[draw opacity=0, fill=blue] (0,1) rectangle (0.25,2);
            \draw[draw opacity=0, fill=blue] (0,0) rectangle (0.25,1);
        \end{scope}
    \end{scope}
    
    \begin{scope}[xshift=5cm+\xsep*2cm]
        \begin{scope}[xshift=-.5cm,yshift=2cm,scale=0.5]
            \draw[gray,thick] (1,1) rectangle (2,2);
            \draw[pink, thick, fill=pink] (1.1,1.1) -- (1.5,1.9) -- (1.9,1.1) -- cycle;
            \draw[black, thick, fill=darkgray] (2,1) rectangle (3,2);
        \end{scope}
        
        \begin{scope}[yscale=0.5]
            \foreach \y in {0,1,2,3} {
                \draw[black, thin] (0,\y) rectangle (1,\y+1);
            }
            \draw[draw opacity=0, fill=blue] (0,3) rectangle (0.32,4);
            \draw[draw opacity=0, fill=blue] (0,2) rectangle (0.32,3);
            \draw[draw opacity=0, fill=blue] (0,1) rectangle (0.04,2);
            \draw[draw opacity=0, fill=blue] (0,0) rectangle (0.32,1);
        \end{scope}
    \end{scope}
    
    \begin{scope}[xshift=5cm+\xsep*3cm]
        \begin{scope}[xshift=-0.5cm,yshift=2.5cm,scale=0.5]
            \draw[gray,thick] (1,1) rectangle (2,2);
            \draw[pink, thick, fill=pink] (1.1,1.1) -- (1.5,1.9) -- (1.9,1.1) -- cycle;
            \draw[black, thick, fill=darkgray] (1,0) rectangle (2,1);
        \end{scope}
    
        \begin{scope}[yscale=0.5]
            \foreach \y in {0,1,2,3} {
                \draw[black, thin] (0,\y) rectangle (1,\y+1);
            }
            \draw[draw opacity=0, fill=blue] (0,3) rectangle (0.32,4);
            \draw[draw opacity=0, fill=blue] (0,2) rectangle (0.32,3);
            \draw[draw opacity=0, fill=blue] (0,1) rectangle (0.32,2);
            \draw[draw opacity=0, fill=blue] (0,0) rectangle (0.04,1);
        \end{scope}
    \end{scope}
    
    \begin{scope}[xshift=5cm+\xsep*4cm]
        \begin{scope}[xshift=-.5cm,yshift=2cm,scale=0.5]
            \draw[gray,thick] (1,1) rectangle (2,2);
            \draw[pink, thick, fill=pink] (1.1,1.1) -- (1.5,1.9) -- (1.9,1.1) -- cycle;
            \draw[->,blue, very thick]  (0.8,1.5) -- (1.5,1.5);
        \end{scope}
    
        \begin{scope}[yscale=0.5]
            \foreach \y in {0,1,2,3} {
                \draw[black, thin] (0,\y) rectangle (1,\y+1);
            }
            \draw[draw opacity=0, fill=blue] (0,3) rectangle (0.04,4);
            \draw[draw opacity=0, fill=blue] (0,2) rectangle (0.32,3);
            \draw[draw opacity=0, fill=blue] (0,1) rectangle (0.32,2);
            \draw[draw opacity=0, fill=blue] (0,0) rectangle (0.32,1);
        \end{scope}
    \end{scope}
    
    \begin{scope}[xshift=5cm+\xsep*5cm]
        \begin{scope}[yshift=2cm,scale=0.5]    
            \node[align=center] at (1,1) {product \\mixing};
        \end{scope}
    
        \begin{scope}[yscale=0.5]
            \foreach \y in {0,1,2,3} {
                \draw[black, thin] (0,\y) rectangle (1,\y+1);
            }
            \draw[draw opacity=0, fill=blue] (0,3) rectangle (0.04,4);
            \node[] at (0.5,3.5) {0.01};
            \draw[draw opacity=0, fill=blue] (0,2) rectangle (0.96,3);
            \node[color=white] at (0.5,2.5) {\textbf{0.97}};
            \draw[draw opacity=0, fill=blue] (0,1) rectangle (0.04,2);
            \node[] at (0.5,1.5) {0.01};
            \draw[draw opacity=0, fill=blue] (0,0) rectangle (0.04,1);
            \node[] at (0.5,0.5) {0.01};
        \end{scope}
    \end{scope}
\end{scope}

\end{tikzpicture}} 
    \hspace{1em}
    \scalebox{0.75}{\begin{tikzpicture}

\node[align=center] at (2,-1em) {(c)};
\node[align=center] at (7,-1em) {(d)};

\newcommand{\wallcell}[2]{\fill[darkgray] (#1,#2) rectangle (#1+1,#2+1)}
\newcommand{\emptycell}[2]{\draw[thick] (#1,#2) rectangle (#1+1,#2+1)}
\newcommand{\applepic}[2]{\draw[solid, fill=darkgreen] (#1+.5,#2+.5) circle (.4)}

\begin{scope}[scale=0.5]
\fill[darkgray] (1,0) rectangle (7,6);
\fill[white] (2,1) rectangle (6,5);
\wallcell{3}{5};
\wallcell{4}{4};
\wallcell{4}{2};
\wallcell{4}{1};
\wallcell{3}{4};
\wallcell{2}{2};
\wallcell{1}{2};

\foreach \x in {1,2,3,4,5,6} {
  \foreach \y in {0,1,2,3,4,5} {
    \emptycell{\x}{\y};
  }
}

\def\xplay{3}
\def\yplay{3}

\draw[red, thick, fill=red] (\xplay+.1,\yplay+.1) -- (\xplay+.5,\yplay+.9) -- (\xplay+.9,\yplay+.1) -- cycle;
\draw[->,blue, very thick]  (\xplay-.2,\yplay+.5) -- (\xplay+.5,\yplay+.5);

\def\xapple{5}
\def\yapple{3}
\applepic{\xapple}{\yapple};
\end{scope}

\begin{scope}[xshift=-.2cm]
    \begin{scope}[xshift=5.1cm]
        \begin{scope}[yscale=0.5,align=left, text width=2cm]
            \node[] at (0,3.5) {Left};
            \node[] at (0,2.5) {Up};
            \node[] at (0,1.5) {Right};
            \node[] at (0,0.5) {Down};
        \end{scope}
    \end{scope}
    
    \begin{scope}[xshift=5cm]
        \begin{scope}[yshift=2.5cm,scale=0.3]
            \draw[gray,thick] (1,1) rectangle (2,2);
            \draw[pink, thick, fill=pink] (1.1,1.1) -- (1.5,1.9) -- (1.9,1.1) -- cycle;
            \emptycell{0}{1};
            \emptycell{1}{0};
            \emptycell{2}{1};
            \emptycell{1}{2};
            \wallcell{1}{2};
        \end{scope}
    
        \begin{scope}[yscale=0.5]
            \foreach \y in {0,1,2,3} {
                \draw[black, thin] (0,\y) rectangle (1,\y+1);
            }
            \draw[draw opacity=0, fill=blue] (0,3) rectangle (0.22,4);
            \draw[draw opacity=0, fill=blue] (0,2) rectangle (0.04,3);
            \draw[draw opacity=0, fill=blue] (0,1) rectangle (0.22,2);
            \draw[draw opacity=0, fill=blue] (0,0) rectangle (0.52,1);
        \end{scope}
    \end{scope}

    \begin{scope}[xshift=6.5cm]
        \begin{scope}[xshift=0cm,yshift=2.5cm,scale=0.3]
            \draw[gray,thick] (1,1) rectangle (2,2);
            \draw[pink, thick, fill=pink] (1.1,1.1) -- (1.5,1.9) -- (1.9,1.1) -- cycle;
            \emptycell{-1}{1};
            \wallcell{-1}{1};
            \emptycell{0}{1};
            \emptycell{2}{1};
            \emptycell{3}{1};
            \applepic{3}{1};
            \emptycell{1}{-1};
            \emptycell{1}{0};
            \emptycell{1}{1};
            \emptycell{1}{2};
            \emptycell{1}{3};
            \emptycell{1}{2};
            \emptycell{1}{3};
            \wallcell{1}{2};
            \wallcell{1}{3};
        \end{scope}
    
        \begin{scope}[yscale=0.5]
            \foreach \y in {0,1,2,3} {
                \draw[black, thin] (0,\y) rectangle (1,\y+1);
            }
            \draw[draw opacity=0, fill=blue] (0,3) rectangle (0.04,4);
            \draw[draw opacity=0, fill=blue] (0,2) rectangle (0.04,3);
            \draw[draw opacity=0, fill=blue] (0,1) rectangle (0.96,2);
            \draw[draw opacity=0, fill=blue] (0,0) rectangle (0.04,1);
        \end{scope}
    \end{scope}

    \begin{scope}[xshift=8cm]
        \begin{scope}[xshift=-.5cm,yshift=2cm,scale=0.5]
            \draw[gray,thick] (1,1) rectangle (2,2);
            \draw[pink, thick, fill=pink] (1.1,1.1) -- (1.5,1.9) -- (1.9,1.1) -- cycle;
            \draw[->,blue, very thick]  (0.8,1.5) -- (1.5,1.5);
        \end{scope}
    
        \begin{scope}[yscale=0.5]
            \foreach \y in {0,1,2,3} {
                \draw[black, thin] (0,\y) rectangle (1,\y+1);
            }
            \draw[draw opacity=0, fill=blue] (0,3) rectangle (0.04,4);
            \draw[draw opacity=0, fill=blue] (0,2) rectangle (0.32,3);
            \draw[draw opacity=0, fill=blue] (0,1) rectangle (0.32,2);
            \draw[draw opacity=0, fill=blue] (0,0) rectangle (0.32,1);
        \end{scope}
    \end{scope}
    
    \begin{scope}[xshift=9.5cm]
        \begin{scope}[yshift=2cm,scale=0.5]    
            \node[align=center] at (1,1) {product \\mixing};
        \end{scope}
    
        \begin{scope}[yscale=0.5]
            \foreach \y in {0,1,2,3} {
                \draw[black, thin] (0,\y) rectangle (1,\y+1);
            }
            \draw[draw opacity=0, fill=blue] (0,3) rectangle (0.04,4);
            \node[] at (0.5,3.5) {0.01};
            \draw[draw opacity=0, fill=blue] (0,2) rectangle (0.04,3);
            \node[] at (0.5,2.5) {0.01};
            \draw[draw opacity=0, fill=blue] (0,1) rectangle (0.96,2);
            \node[color=white] at (0.5,1.5) {\textbf{0.96}};
            \draw[draw opacity=0, fill=blue] (0,0) rectangle (0.04,1);
            \node[] at (0.5,0.5) {0.02};
        \end{scope}
    \end{scope}
\end{scope}

\end{tikzpicture}}
    \end{center}
    \caption{
    (a) A simple maze environment. The dark gray cells are walls, the green circle
    is a goal.
    The blue arrow symbolizes the fact that the agent (red triangle) is coming from the left. 
    (b) A simple context model with five \mutexsets{}:
    One \mutexset{} for each of the four cells around the triangle,
    and one \mutexset{} for the last chosen action.
    Each of the first four \mutexsets{} contains three contexts (wall, empty cell, goal),
    and the last \mutexset{} contains four contexts (one for each action).
    The 5 active contexts (one per \mutexset{}) for the situation shown in (a) are depicted at the top, while their individual probability predictions are the horizontal blue bars for each of the four actions.
    The last column is the resulting product mixing prediction of the 5 predictions.
    No individual context prediction exceeds 1/3 for any action, yet the product mixing 
    prediction is close to 1 for the action Up.
    (c) Another situation.
    (d) A different set of \mutexsets{} for the situation in (c):
    A 1-cross around the agent, a 2-cross around the agent, and the last action.
    The specialized 2-cross context is certain that the correct action is Right,
    despite the two other contexts together giving more weight to action Down.
    The resulting product mixing gives high probability to Right,
    showing that, in product mixing, specialized contexts can take precedence over less-certain more-general contexts.
    }
    \label{fig:product_mixing}
\end{figure*}
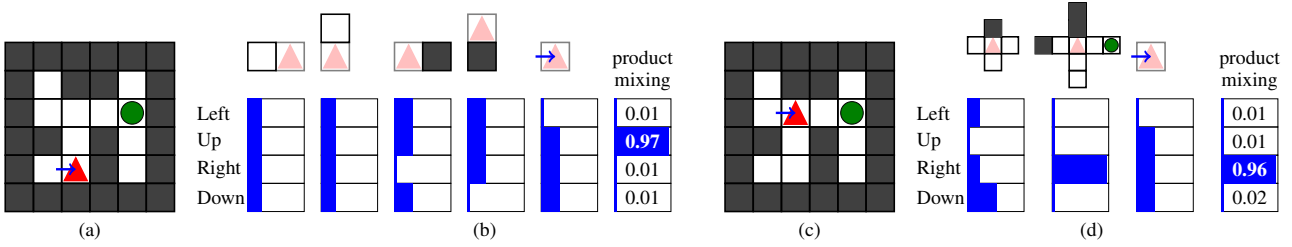

\begin{remark}
A predictor that has a uniform distribution
has no effect within a product mixture.
Hence, adding new predictors initialized with uniform predictions does not change the policy,
and similarly, if a context does not happen to be useful to learn a good policy,
optimization will push its weights toward the uniform distribution, implicitly discarding it.
\end{remark}

Hence, product mixing is able to take advantage of both general contexts that occur in many situations and specialised contexts tailored to specific situations --- and anything in-between.

Our LTS with context models algorithm  is given in \cref{alg:ltscm},
building upon the one by \citet{orseau2021policy} with a few differences.
As mentioned earlier, it is a best-first search algorithm 
and uses a priority queue to maintain the nodes to be expanded next.
\todo{Rename to \code{"budget_used"}, also in appendix}
It is also budgeted and returns \code{"budget_reached"}
if too many nodes have been expanded.
It returns \code{"no_solution"} if all nodes have been expanded
without reaching a solution node --- assuming safe pruning or no pruning.
Safe pruning (using \code{visited_states})
can be performed 
if the policy is Markovian~\protect\cite{orseau2018single},
which is the case in particular when the set of active contexts
$\contextset(n)$ depends only on \code{state}$(n)$.
The algorithm assumes the existence of application-specific \code{state} and \code{state_transition} functions,
such that \code{state}($n'$) = \code{state_transition(state}$(n), a(n'))$
for all $n'\in\children(n)$.
Note that with context models the prediction $\pol(n'\mid n)$ depends
on the active contexts $\contextset(n)$ but \emph{not} on the state of a child node.
This allows us to delay the state transition until the child is extracted from the queue, saving up to a branching factor of state transitions (see also
\cite{agostinelli2021expansions}).

\begin{algorithm}[tbh!]
\caption{Budgeted LTS with context models.
Returns a solution node if any is found, 
or \code{"budget_reached"} or \code{"no_solution"}.}
\label{alg:ltscm}
\begin{lstlisting}
# $n^0$: root node
# $B$: node expansion budget
# $\beta$: parameters of the context models
def LTS+CM($n^0$, $B$, $\beta$):
  q = priority_queue()
  # tuple: {$\lstcommentcolor{\dop,\, d,\, \pol_n,}$ node, state, action}
  tup = {0, 0, 1, $n^0$, state($n^0$), False}
  q.insert(tup) # insert root node/state
  visited_states = {} # dict: state($\lstcommentcolor{n}$) -> $\lstcommentcolor{\pol(n)}$
  repeat forever:
    if q is empty: return "no_solution"
    # Extract the tuple with minimum cost $\lstcommentcolor{\dop}$
    $\dop n$, d, $\pol_{n}$, n, s_parent, a = q.extract_min()
    if $n\in\nodeset^*$: return n # solution found
    s = state_transition(s_parent, a) if a else s_parent
    $\pol_s$ = visited_states.get(s, default=0)
    # Note: BFS ensures $\lstcommentcolor{\dop(n_s) \leq \dop(n); s=state(n_s)}$ 
    # Optional: Prune the search if s is better
    if $\pol_s \geq \pol_n$: continue 
    else: visited_states.set(s, $\pol_n$)
    # Node expansion
    expanded += 1
    if expanded == $B$: return "budget_reached"
    
    Z = $\sum_{a\in\actionset(n)} \prod_{c\in\contextset(n)} p_c(a; \beta)$ # normalizer
    for $n'\in\children(n)$:
      a = $a(n')$ # action
      # Product mixing of the active contexts' predictions
      $\prodmix{_{,a}}$ = $\frac1Z \prod_{c\in\contextset(n)} p_c($a$; \beta)$ # See (*\lstcommentcolor{\cref{eq:pcabeta}}*)
      # Action probability, $\lstcommentcolor{\epsmix}$ ensures $\lstcommentcolor{\pol_{n'}>0}$ 
      $\pol_{n'}$ = $\pol_{n}((1-\epsmix)\prodmix{_{,a}} + \frac{\epsmix}{|\actionset(n)|})$
      q.insert({(d+1)/$\pol_{n'}$, d+1, $\pol_{n'}$, $n'$, s, a})
      
\end{lstlisting}
\end{algorithm}

\begin{remark}
In practice, usually a \mutexset{} can be implemented as a hashtable
as for pattern databases:
the active context is read from the current state of the environment,
and the corresponding predictor is retrieved from the hashtable.
This allows for a computational cost of $O(\log |M|)$ per mutex set $M$,
or even $O(1)$ with perfect hash functions,
and thus $O(\sum_{M\in\patternset}\log |M|)$ which is much smaller than $|\contextset|$.
Using an imperfect hashtable, only the contexts that appear on the paths to the found solution nodes need to be stored.
\end{remark}

\section{Convexity}\label{sec:Conv}

Because the LTS loss in \cref{eq:loss} is different from the log loss~\cite{cesabianchi2006prediction} (due to the sum in-between the products),  
optimization does \emph{not} reduce to maximum likelihood estimation.
However, we show that convexity in the log loss implies convexity in the LTS loss.
This means, in particular, that if a probability distribution is log-concave (such as all the members of the exponential family), that is, the log loss for such models is convex,
then the LTS loss is convex in these parameters, too.

First we show that every sequence of functions with a convex log loss
also have convex \emph{inverse} loss and LTS loss.

\begin{theorem}[Log loss to inverse loss convexity]\label{thm:logloss_to_inverseloss}
Let $f_1, f_2,\dots f_s$ be a sequence of positive functions
with $f_i: \Reals^n\to(0, \infty)$ for all $i\in[s]$
and such that
$\beta\mapsto -\log f_i(\beta)$ 
is convex for each $i\in[s]$, then
$L(\beta) = \sum_k \frac{1}{\prod_t f_{k,t}(\beta)}$ is convex, where each $(k,t)$ corresponds to a unique index in $[s]$.
\end{theorem}
The proof is in \iflong{\cref{apdx:logloss_to_ltsloss}}{Appendix E.1}.
For a policy $\pol(\cdot ; \beta)$ parameterized by $\beta$,
the LTS loss in \cref{eq:loss} 
is $L_{\nodeset'}(\beta)=\sum_{k\in\nodeset'}d(n^k)/\pol(n^k;\beta)$,
and its convexity follows from \cref{thm:logloss_to_inverseloss} by taking $f_{k,0}(\cdot)=1/d(n^k)$,
and $f_{k,t}(\beta)=\pol(n^k_{[t]}| n^k_{[t-1]}; \beta)$
such that $\prod_{t=1}^{d(n^k)}f_{k,t}(\beta) = \pol(n^k; \beta)$.

\Cref{thm:logloss_to_inverseloss} means that many tools of compression and online prediction in the log loss can be transferred to the LTS loss case.
In particular, when there is only one \mutexset{} ($|\patternset|=1$),
the $f_i$ are simple categorical distributions, that is, $f_i(\beta) = \beta_{j_t}$ for some index $j_t$,
and thus $-\log f_i$ is a convex function, so the corresponding LTS loss is convex too.
Unfortunately, the LTS loss function for such a model is convex in $\beta$ only when there is only one mutex set, $|\patternset|=1$.
Fortunately, it becomes convex for $|\patternset|\geq 1$ when we reparameterize the context predictors with $\beta \leadsto \exp \beta$.

Let $\beta_{c, a}\in[\ln\epslow, 0]$ be the value of the parameter of the predictor
for context $c$ for the edge label $a$.
Then the prediction of a context $c$ is defined as
\begin{align}\label{eq:pcabeta}
    \forall a\in\actionset(n): p_c(a; \beta) = \frac{\exp(\beta_{c, a})}{\sum_{a' \in\actionset(n)} \exp(\beta_{c, a'})}\,.
\end{align}
We can also now make precise the definition of $\betaset$:
$\betaset=[\ln\epslow, 0]^{|\contextset|\times A}$,
and note that $p_c(a; \beta) \geq \epslow / |\actionset(n)|$. 
Similarly to geometric mixing~\cite{mattern2013geomix,mattern2016phd}, it can be proven that context models have a convex log loss, and thus their LTS loss is also convex by \cref{thm:logloss_to_inverseloss}.
In \iflong{\cref{sec:LTSconv}}{Appendix E.2} we provide a more direct proof,
and a generalization to the exponential family for finite sets of actions.

Plugging \eqref{eq:pcabeta} into \cref{eq:prodmix} and pushing the probabilities away from 0 with $\epsmix >0$~\cite{orseau2018single} we obtain the policy's probability for a child $n'$ of $n$ (\ie for the action $a(n')$ at node $n$)
with parameters $\beta$:
\begin{align}
    \prodmix(n, a;\beta) &= 
    \frac{\exp(\sum_{c\in\contextset(n)} \beta_{c,a})}{\sum_{a'\in\actionset(n)}\exp\left(\sum_{c\in\contextset(n)} \beta_{c, a'}\right)}\,, \label{eq:prodmix_softmax}\\
    \pol(n'\mid n; \beta) &= (1-\epsmix)\prodmix(n, a(n');\beta) + \frac{\epsmix}{|\actionset(n)|}\,. \label{eq:pol_beta}
\end{align}

\subsection{Optimization}\label{sec:Opt}

We can now give a more explicit form of the LTS loss function of \cref{eq:loss} for context models 
with a dependency on the parameters $\beta$,
for a set of solution nodes $\nodeset'$ assumed to all have different roots:
\begin{align}
    \Loss(\nodeset', \beta) &= 
    \sum_{n\in\nodeset'} \loss(n, \beta)\,,  \label{eq:ltsloss_cm}\\
    \loss(n, \beta) &= \frac{d(n)}{\pol(n ; \beta)}
    ~=~ \frac{d(n)}{\prod_{j=0}^{d(n)-1} \pol(n_{[j+1]}|n_{[j]} ; \beta)} \label{eq:ltsinstantloss_cm}\\
    = d(n)&\prod_{j=0}^{d(n)-1} \!\sum_{a'\in\actionset(n_{[j]})}\!\!\!\!\exp\left(\sum_{c\in\contextset(n_{[j]})} \beta_{c, a'} - \beta_{c, a(n_{[j+1]})}\right)\notag
\end{align}
where $a(n_{[j+1]})$ should be read as the action chosen at step $j$,
and the last equality follows from  \cref{eq:prodmix_softmax,eq:pol_beta}
where we take $\epsmix=0$ during optimization.
Recall that this loss function $\Loss$ gives an upper bound on the total search time (in node expansions) required for LTS to find all the solutions $\nodeset'$
for their corresponding problems (root nodes),
and thus optimizing the parameters corresponds to optimizing the search time.

\subsection{Online Search-and-Learn Guarantees}

Suppose that at each time step $t=1, 2\dots$, the learner receives
a problem $n^0_t$ (a root node)
and uses LTS with parameters $\beta^t\in\betaset$ until it finds a solution node $n_t\in\nodeset^*(n^0_t)$.
The parameters are then updated using $n_t$ (and previous nodes)
and the next step $t+1$ begins.

Let $\nodeset_t= (n_1, \dots, n_t)$ be the sequence of found solution nodes.
For the loss function of \cref{eq:ltsloss_cm},
after $t$ found solution nodes, the optimal parameters \emph{in hindsight}
are $\beta^*_t = \argmin_{\beta\in\betaset} \Loss(\nodeset_t, \beta)$.
We want to know how the learner fares against $\beta^*_t$ --- which is a moving target as $t$ increases.
The \emph{regret}~\cite{hazan2016oco} at step $t$ is the cumulative difference
between the loss incurred by the learner
with its time varying parameters $\beta^i, i=1,2,\dots, t$,
and the loss when using the optimum parameters in hindsight $\beta^*_t$:
\begin{align*}
    \regret(\nodeset_t) = \sum_{i\in[t]} \loss(n_i, \beta^i)  -  L(\nodeset_t,\beta^*_t)\,.
\end{align*}

A straightforward implication of the convexity of \cref{eq:ltsinstantloss_cm} is that we can use 
Online Gradient Descent (OGD) \cite{Zin03} or some of its many variants such as Adagrad~\cite{duchi2011adaptive} and ensure that the algorithm incurs a regret of $\regret(\nodeset_t)=O(|\actionset|\,|\contextset|G\sqrt{t}\ln\frac{1}{\epslow})$,
where $G$ is the largest observed gradient in infinite norm
\footnote{The dependency on the largest gradient can be softened significantly,
\eg with Adagrad and sporadic resets of the learning rates.}
and when using quadratic regularization.
Regret bounds are related to the learning speed (the smaller the bound, the faster the learning),
that is, roughly speaking, how fast the parameters converge to their optimal values for the same sequence of solution nodes.
Such a regret bound (assuming it is tight enough) also allows to observe the impact of the different quantities on the regret, such as the number of contexts $|\contextset|$, or $\epslow$.

OGD and its many variants are computationally efficient as they take $O(d(n)|\actionset|\,|\patternset|)$ computation time per solution node $n$,
but they are not very data efficient, due to the \emph{linearization} of the loss function ---
the so-called `gradient trick'~\cite{cesabianchi2006prediction}.
To make the most of the data, we avoid linearization by sequentially minimizing the full regularized loss function $\Loss(\nodeset_t,\cdot)+R(\cdot)$ where $R(\beta)$ is a convex regularization function. That is, at each step, we set:
\begin{align}\label{eq:beta_full_update}
    \beta^{t+1} = \argmin_{\beta\in\betaset}\Loss(\nodeset_t,\beta)+R(\beta)
\end{align}
which can be solved using standard convex optimization techniques (see \iflong{\cref{apdx:convex_optim}}{Appendix C})~\cite{boyd2004convex}.
This update is known as (non-linearized) Follow the Leader (FTL)
which automatically adapts to local strong convexity and has a fast $O(\log T)$ regret without tuning a learning rate~\cite{shalev2007online}, except that we add regularization to avoid overfitting
which FTL suffers from.
Unfortunately, solving \cref{eq:beta_full_update} even approximately at each step is too computational costly, so we amortize this cost by delaying updates (see below), which of course incurs a learning cost, \eg \cite{joulani2013delayed}.

\section{Experiments}\label{sec:Exp}

As with previous work,
in the experiments we use the LTS algorithm with context models (\cref{alg:ltscm}) within the search-and-learn loop of the Bootstrap process~\cite{ArfaeeZH11}
to solve a dataset of problems,
then test the learned model on a separate test set.
See \iflong{\cref{apdx:bootstrap}}{Appendix A} for more details.
Note that the Bootstrap process is a little different from the online learning setting,
so the theoretical guarantees mentioned above may not carry over strictly --- this analysis is left for future work.,

This allows us to compare LTS with context models (LTS+CM) in particular
with previous results using LTS with neural networks (LTS+NN)
\cite{guez2019planning,orseau2021policy} on three domains.
We also train LTS+CM to solve the Rubik's cube and compare with other approaches.

\newif\iffullresults
\fullresultsfalse

\newcommand{\resultstable}{
\begin{table*}[tb!]
\sisetup{detect-all=true,group-minimum-digits=3,table-align-text-before = false}
    \centering
    \begin{tabular}{llS[table-format=3.2]S[table-format=3.1]S[table-format=7.1]S[table-format=5]}
    \toprule
    Domain & Algorithm &{\%solved} & {Length} & {Expansions} & {Time (ms)} \\
\midrule
\midrule
    Boxoban
           &LTS+CM (this work)                 & 100.00 & 41.7 & \iffullresults\else\bfseries\fi 2132.3 & 124  \\
           &LTS+NN~\protect\cite{orseau2021policy}  & 100.00 & \iffullresults\else\bfseries\fi 40.1 & 2640.4 & 19500  \\
        \iffullresults
           &PHS*~\protect\cite{orseau2021policy}     &    100.00 &      37.6 &     \bfseries 1522.1 & 11300 \\
           &WA*, w=1.5~\protect\cite{orseau2021policy} &    100.00 &      \bfseries 34.5 &     3729.1 & 25500 \\
           \cmidrule{2-6}
           &LTS+CM (this work) @500k                 & 100.00 & 48.5 & \bfseries 858.1 & 55  \\
           &DeepCubeA~\protect\cite{agostinelli2019rubik} & 100.00 & \bfseries 32.9 & 1050.0 & 2350  \\    
        \fi
    \midrule
    The Witness 
            &LTS+CM (this work)                         & \iffullresults\bfseries\fi 100.00 & 15.5 & \bfseries 102.8   & 9  \\
            &LTS+NN~\protect\cite{orseau2021policy}     & \iffullresults\bfseries\fi 100.00 & \bfseries 14.8 & 520.2 & 3200   \\
        \iffullresults
           &PHS*~\protect\cite{orseau2021policy}        & \iffullresults\bfseries\fi 100.00 &      15.0 &      408.1 &  3000 \\
            &WA*, w=1.5~\protect\cite{orseau2021policy} &     99.90 &      {\textit{\hspace{1ex}14.6}} &    \multicolumn{1}{r}{\textit{18\,345.2}} & {\textit{71500}} \\
        \fi
    \midrule
    STP (24-puzzle) 
            &LTS+CM (this work)   & \bfseries 100.00  & 211.2 & \iffullresults\else\bfseries\fi 5667.4 & 236  \\
            &LTS+NN~\protect\cite{orseau2021policy}   & 0.90 & {\itshape 145.1} &
            \multicolumn{1}{r}{\hfill\itshape 39\,005.6}
            & \multicolumn{1}{r}{\hfill\itshape 31\,100\phantom{.0}} \\
        \iffullresults
            &PHS*~\protect\cite{orseau2021policy}       &   \bfseries 100.00 &     224.0 &     2867.2 &  2800 \\
            &WA*, w=1.5~\protect\cite{orseau2021policy} &   \bfseries 100.00 &     129.8 & \bfseries 1989.8 &  1600 \\
           \cmidrule{2-6}
           &DeepCubeA~\protect\cite{agostinelli2019rubik} & \bfseries 100.00 & \bfseries 89.5 & 6440000.0 & 19330  \\    
        \fi
\midrule
    \midrule
    Boxoban hard
           &LTS+CM (this work)         & \bfseries 100.00 & \iffullresults\bfseries\fi 67.8 & 48058.6 & 3275 \\
    \iffullresults
           &LTS+CM (this work) @500k   & \bfseries 100.00 & 72.5 & \bfseries \bfseries 12166.2 & 761 \\
    \fi
           &LTS+NN \protect\cite{guez2019planning} & 94.00 & \multicolumn{1}{r}{n/a\phantom{.}}  & \multicolumn{1}{r}{n/a} & {\hfill\textit{3\,600}\phantom{.}} \\
           &ExPoSe~\protect\cite{mittal2022expose} & 97.30 & \multicolumn{1}{r}{n/a\phantom{.}}  & \multicolumn{1}{r}{n/a}       & \multicolumn{1}{r}{n/a\phantom{.0}}  \\
    \midrule
    \iffullresults
    Rubik's cube 
           &LTS+CM (this work) @300k   & 100.00 & 39.8 &  247862.4 & 18949 \\
           &LTS+CM (this work) @400k   & 100.00 & 53.3 &   20647.9 &   950 \\
        \fi
    \iffullresults\else Rubik's cube\fi
           &LTS+CM (this work) \iffullresults@5M\fi     & 100.00 & 81.7 &  \bfseries 498.0 &  16 \\
           &DeepCubeA~\protect\cite{agostinelli2019rubik} & 100.00 & \bfseries 21.5 & {$\sim$}600000.0 & 24220  \\
        \iffullresults
           \cmidrule{2-6}
           &LTS+CM (this work) \iffullresults@5M\fi     & 100.00 & \bfseries 78.6 & \bfseries 431.7 & 16 \\
        \fi
           &GBFS(A+M)~\protect\cite{allen2021focused_macros} & 100.00 & 378.0 & {\textdagger}171300.0 & \multicolumn{1}{r}{n/a\phantom{.0}} \\
    \bottomrule
\end{tabular}
    \caption{
    \iffullresults
    More test results.
    See \cref{tab:results} and the text in \cref{apdx:results} for more information.
    The line splits in Boxoban and STP are because the second group uses different training sets from the rest.    
    The test set used by DeepCubeA for STP is different from that of LTS+\{CM,NN\},
    but we expect the comparison to be meaningful anyway.
    \textdagger Does not account for the cost of macro-actions.
    \else
    Results on the test sets.
    The last 3 columns are the averages over the test instances.
    The first three domains allow for a fair comparison between LTS with context models 
    and LTS with neural networks~\protect\cite{orseau2021policy}
    using the same 50k training instances and initial budget.
    For the last two domains, comparison to prior work is more cursory and is provided for information only, in particular because the objective of DeepCubeA is to provide near-optimal-length solutions rather than fast solutions.
    The values for LTS+\{CM,NN\} all use a single CPU, no GPU (except for LTS+NN~\protect\cite{guez2019planning}).
    DeepCubeA uses four high-end GPU cards.
    More results can be found in \iflong{\cref{tab:fullresults}}{Table 2} in \iflong{\cref{apdx:results}}{Appendix H}.
    \textdagger Does not account for the cost of macro-actions.
    \fi
    }
    \iffullresults\label{tab:fullresults}\else\label{tab:results}\fi
\end{table*}
}

\resultstable

\paragraph{LTS+NN's domains.}
We foremost compare LTS with context models (LTS+CM) with LTS with a convolutional neural network~\cite{orseau2021policy} (LTS+NN) on the three domains where the latter was tested:
(a) Sokoban (Boxoban)~\cite{boxobanlevels} on the standard 1000 test problems, a PSPACE-hard puzzle~\cite{Culberson1999} where the player must push boxes onto goal positions while avoiding deadlocks, 
(b) The Witness, a color partitioning problem that is NP-hard in general~\cite{abel2020witness}, and 
(c) the 24 ($5\times 5$) sliding-tile puzzle (STP), a sorting problem on a grid, for which finding short solutions is also NP-hard~\cite{ratner1986puzzle}.
As in previous work, we train LTS+CM on the same datasets of \num{50000} problems each, with the same initial budget (\num{2000} node expansions for Sokoban and The Witness, \num{7000} for STP)
and stop as soon as the training set is entirely solved.
Training LTS+CM for these domains took less than 2 hours each.

\paragraph{Harder Sokoban.}

Additionally, we compare algorithms on the Boxoban `hard' set of \num{3332} problems.
\citet{guez2019planning} trained 
a convLSTM network on the medium-difficulty dataset (450k problems) with a standard actor-critic setup --- not the LTS loss --- and used LTS (hence LTS+NN) at test time.
The more recent ExPoSe algorithm~\cite{mittal2022expose} updates the parameters of a policy neural network
\footnote{The architecture of the neural network was not specified.}
\emph{during} the search, and is trained on both the medium set (450k problems) and the `unfiltered' Boxoban set (900k problems) with solution trajectories obtained from an A* search.

\paragraph{Rubik's Cube.}

We also use LTS+CM to learn a fast policy for the Rubik's cube, with an initial budget of $B_1=21000$. 
We use a sequence of datasets containing  100k problems each,
generated with a random walk of between $m$ and $m'=m+5$ moves from the solution,
where $m$ increases by steps of 5 from 0 to 50, after which 
we set $m'=m=50$ for each new generated set.
DeepCubeA~\cite{agostinelli2019rubik} uses a fairly large neural network to learn 
in a supervised fashion from trajectories generated with a backward model of the environment,
and Weighted A* is used to solve random test cubes.
Their goal is to learn a policy that returns solutions of near-optimal length.
By contrast, our goal is to learn a fast-solving policy.
\citet{allen2021focused_macros} takes a completely different approach (no neural network) by learning a set of `focused macro actions' which are meant to change the state as little as possible
so as to mimic the so-called `algorithms' that human experts use to solve the Rubik's cube. 
They use a rather small budget of 2 million actions to learn the macro actions, 
but also use the more informative goal-count scoring function (how many variables of the state have the correct value), while we only assume access to the more basic solved/unsolved  function.
As with previous work, we report solution lengths in the quarter-turn metric.
Our test set contains 1000 cubes scrambled 100 times each --- this is likely
more than enough to generate random cubes~\cite{korf1997cube} --- and we expect the difficulty
to match that of previous work.

\paragraph{Machine description.}
We used a single EPYC 7B12 (64 cores, 128 threads) server with 512GB of RAM without GPU.
During training and testing, 64 problems are attempted concurrently --- one problem per CPU core.
Optimization uses 64 threads to calculate the loss, gradient and updates.

\paragraph{Hyperparameters.}
For all experiments we use $\epslow=10^{-4}$,
$\epsmix=10^{-3}$, a quadratic regularization $R(\beta)=5\|\beta-\beta_0\|^2$
where $\beta_0=(1-1/A)\ln\epslow$
(see \iflong{\cref{apdx:betasimplex}}{Appendix F}).
The convex optimization algorithm we use to solve \cref{eq:beta_full_update} is detailed in \iflong{\cref{apdx:convex_optim}}{Appendix C}.

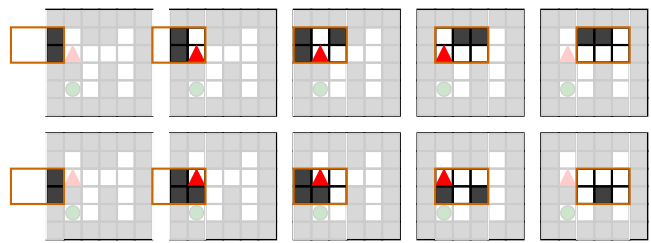
\begin{figure}[tbh!]
    \centering
    \begin{tikzpicture}
\newcommand{\wallcell}[2]{\fill[darkgray] (#1,#2) rectangle (#1+1,#2+1)}
\newcommand{\emptycell}[2]{\draw[thick] (#1,#2) rectangle (#1+1,#2+1)}
\newcommand{\applepic}[2]{\draw[solid, fill=darkgreen] (#1+.5,#2+.5) circle (.4)}

\newcommand{\grid}[4]{
\begin{scope}[scale=0.5]
\fill[darkgray] (0,0) rectangle (6,6);
\fill[white] (1,1) rectangle (5,5);
\wallcell{2}{5};
\wallcell{3}{4};
\wallcell{3}{2};
\wallcell{3}{1};
\wallcell{2}{4};
\wallcell{1}{2};
\wallcell{0}{2};

\foreach \x in {0,1,2,3,4,5} {
  \foreach \y in {0,1,2,3,4,5} {
    \emptycell{\x}{\y};
  }
}

\def\xplay{1}
\def\yplay{3}

\draw[red, thick, fill=red] (\xplay+.1,\yplay+.1) -- (\xplay+.5,\yplay+.9) -- (\xplay+.9,\yplay+.1) -- cycle;

\def\xapple{1}
\def\yapple{1}
\applepic{\xapple}{\yapple};
\fill[white,opacity=0.8] (0,0) rectangle (#1,6);
\fill[white,opacity=0.8] (#3+1,0) rectangle (6,6);
\fill[white,opacity=0.8] (#1,0) rectangle (#3+1,#2);
\fill[white,opacity=0.8] (#1,#4+1) rectangle (#3+1,6);
\draw[darkorange,thick] (#1,#2) rectangle (#3+1,#4+1);

\end{scope}
}

\def\xsep{3.5}
\def\ysep{-3.5}

\begin{scope}[xshift=0,scale=0.47]
    \begin{scope}[yshift=\ysep*0cm]
        \grid{-2}{3}{0}{4}
        \begin{scope}[xshift=\xsep*1cm]
        \grid{-1}{3}{1}{4}
        \end{scope}
        \begin{scope}[xshift=\xsep*2cm]
        \grid{0}{3}{2}{4}
        \end{scope}
        \begin{scope}[xshift=\xsep*3cm]
        \grid{1}{3}{3}{4}
        \end{scope}
        \begin{scope}[xshift=\xsep*4cm]
        \grid{2}{3}{4}{4}
        \end{scope}
    \end{scope}

    \begin{scope}[yshift=\ysep*1cm]
        \grid{-2}{2}{0}{3}
        \begin{scope}[xshift=\xsep*1cm]
        \grid{-1}{2}{1}{3}
        \end{scope}
        \begin{scope}[xshift=\xsep*2cm]
        \grid{0}{2}{2}{3}
        \end{scope}
        \begin{scope}[xshift=\xsep*3cm]
        \grid{1}{2}{3}{3}
        \end{scope}
        \begin{scope}[xshift=\xsep*4cm]
        \grid{2}{2}{4}{3}
        \end{scope}
    \end{scope}

\end{scope}

\end{tikzpicture}
    \caption{Example of a relative tiling
    of row span 2, column span 3, at maximum row distance 1 and maximum column distance 3 around the agent (red triangle).
    Each orange rectangle is a mutex set of at most $4^{6}$ different contexts.
    A padding value can be chosen arbitrarily (such as the wall value) for
    cells outside the grid.}
    \label{fig:tiling}
\end{figure}

\paragraph{Mutex sets.}
For Sokoban, STP, and The Witness we use
several mutex sets of
rectangular shapes at various distances around the agent
(the player in Sokoban, the tip of the `snake' in The Witness, the blank in STP),
which we call \emph{relative tilings}.
An example of relative tiling is given in \cref{fig:tiling},
and a more information can be found in \iflong{\cref{apdx:relative_tilings}}{Appendix G}.
For the Rubik's cube, each \mutexset{} $\{i, j\}$ corresponds to the ordered colors of the two cubies (the small cubes that make up the Rubik's cube) at location $i$ and $j$ (such as the up-front-right corner 
and the back-right edge).
There are 20 locations, hence 190 different \mutexsets{}, and each of them contains at most $24^2$ contexts (there are 8 corner cubies, each with 3 possible orientations, and 12 side cubies, each with 2 possible orientations). 
For all domains, to these mutex sets we add one mutex set for the last action, indicating the action the agent performed to reach the node;
for Sokoban this includes whether the last action was a push. The first 3 domains all have 4 actions (up, down, left, right), and the Rubik's cube has 12 actions (a rotation of each face, in either direction).

\paragraph{Results.}
The algorithms are tested on test sets that are separate from the training sets, see \cref{tab:results}.
For the first three domains, LTS+CM performs better than LTS+NN, even solving all test instances of the STP while LTS+NN solves less than 1\% of them.
On The Witness, LTS+CM learns a policy that allows it to expand 5 times fewer nodes than LTS+NN.
LTS+CM also solves all instances of the Boxoban hard set, by contrast to previous published work, and despite being trained only on 50k problems.
On the Rubik's cube,
LTS+CM learns a policy that is hundreds of times faster than previous work --- though recall that DeepCubeA's objective of finding short solutions differs from ours.
This may be surprising given how simple the contexts are --- each context `sees' only two cubies --- and is a clear sign that product mixing is taking full advantage of the learned individual context predictions.

\section{Conclusion}\label{sec:Conc}

We have devised a parameterized policy for the Levin Tree Search (LTS) algorithm
using product-of-experts of context models that ensures that the LTS loss function is convex.
While neural networks --- where convexity is almost certainly lost --- have achieved impressive results recently, we show that our algorithm is competitive with published results, if not better.

Convexity allows us in particular to use convex optimization algorithms
and to provide regret guarantees in the online learning setting. 
While this provides a good basis to work with, 
this notion of regret holds against any competitor that learns from the same set of \emph{solution} nodes.
The next question is how we can obtain an online search-and-learn regret guarantee against a competitor for the same set of \emph{problems} (root nodes),
for which the cumulative LTS loss is minimum across all sets of solution nodes for the same problems.
And, if this happens to be unachievable, what intermediate regret setting could be considered?
We believe these are important open research questions to tackle.

We have tried to design \mutexsets{} that use only basic domain-specific knowledge (the input representation of agent-centered grid-worlds, or the cubie representation of the Rubik's cube), but in the future it would be interesting to also \emph{learn to search} the space of possible context models --- this would likely require more training data.

LTS with context models, as presented here, cannot directly make use
of a value function or a heuristic function, however they could
either be binarized into multiple mutex sets,
or be used as in PHS*~\cite{orseau2021policy} to estimate the LTS cost at the solution,
or be used as features since the loss function would still be convex (see \iflong{\cref{apdx:convex_optim}}{Appendix C}).

\appendix

\section*{Acknowledgments}
We would like to thank the following people for their useful help and feedback:
 Csaba Szepesvari,
 Pooria Joulani,
 Tor Lattimore,
 Joel Veness,
 Stephen McAleer.

The following people also helped with Racket-specific questions:
 Matthew Flatt,
 Sam Tobin-Hochstadt,
 Bogdan Popa,
 Jeffrey Massung,
 Jens Axel S{\o}gaard,
 Sorawee Porncharoenwase,
 Jack Firth,
 Stephen De Gabrielle,
 Alex Hars\'anyi,
 Shu-Hung You,
 and the rest of the quite helpful and reactive Racket community.

This research was supported by Canada's NSERC and the CIFAR AI Chairs program.

\bibliographystyle{named}
\bibliography{biblio}

\begin{thebibliography}{}

\bibitem[\protect\citeauthoryear{Abel \bgroup \em et al.\egroup
  }{2020}]{abel2020witness}
Zachary Abel, Jeffrey Bosboom, Michael~J. Coulombe, Erik~D. Demaine, Linus
  Hamilton, Adam Hesterberg, Justin Kopinsky, Jayson Lynch, Mikhail Rudoy, and
  Clemens Thielen.
\newblock Who witnesses the witness? finding witnesses in the witness is hard
  and sometimes impossible.
\newblock {\em Theor. Comput. Sci.}, 839:41--102, 2020.

\bibitem[\protect\citeauthoryear{Agostinelli \bgroup \em et al.\egroup
  }{2019}]{agostinelli2019rubik}
Forest Agostinelli, Stephen McAleer, Alexander Shmakov, and Pierre Baldi.
\newblock Solving the rubik’s cube with deep reinforcement learning and
  search.
\newblock {\em Nature Machine Intelligence}, 1, 07 2019.

\bibitem[\protect\citeauthoryear{Agostinelli \bgroup \em et al.\egroup
  }{2021}]{agostinelli2021expansions}
Forest Agostinelli, Alexander Shmakov, Stephen McAleer, Roy Fox, and Pierre
  Baldi.
\newblock A* search without expansions: Learning heuristic functions with deep
  q-networks, 2021.
\newblock arXiv 2102.04518.

\bibitem[\protect\citeauthoryear{Allen \bgroup \em et al.\egroup
  }{2021}]{allen2021focused_macros}
Cameron Allen, Michael Katz, Tim Klinger, George Konidaris, Matthew Riemer, and
  Gerald Tesauro.
\newblock Efficient black-box planning using macro-actions with focused
  effects.
\newblock In {\em Proceedings of the Thirtieth International Joint Conference
  on Artificial Intelligence}, pages 4024--4031, 2021.

\bibitem[\protect\citeauthoryear{Boyd and Vandenberghe}{2004}]{boyd2004convex}
Stephen Boyd and Lieven Vandenberghe.
\newblock {\em Convex Optimization}.
\newblock Cambridge University Press, Cambridge, England, 2004.

\bibitem[\protect\citeauthoryear{B{\"u}chner \bgroup \em et al.\egroup
  }{2022}]{buchner2022a}
Clemens B{\"u}chner, Patrick Ferber, Jendrik Seipp, and Malte Helmert.
\newblock A comparison of abstraction heuristics for rubik's cube.
\newblock In {\em ICAPS 2022 Workshop on Heuristics and Search for
  Domain-independent Planning}, 2022.

\bibitem[\protect\citeauthoryear{Budden \bgroup \em et al.\egroup
  }{2020}]{budden2020gated}
David Budden, Adam Marblestone, Eren Sezener, Tor Lattimore, Gregory Wayne, and
  Joel Veness.
\newblock Gaussian gated linear networks.
\newblock In {\em Advances in Neural Information Processing Systems},
  volume~33, pages 16508--16519. Curran Associates, Inc., 2020.

\bibitem[\protect\citeauthoryear{Cesa-Bianchi and
  Lugosi}{2006}]{cesabianchi2006prediction}
Nicolo Cesa-Bianchi and Gabor Lugosi.
\newblock {\em Prediction, Learning, and Games}.
\newblock Cambridge University Press, New York, NY, USA, 2006.

\bibitem[\protect\citeauthoryear{Cortés}{2006}]{cortes2006ngd}
Jorge Cortés.
\newblock Finite-time convergent gradient flows with applications to network
  consensus.
\newblock {\em Automatica}, 42(11):1993--2000, 2006.

\bibitem[\protect\citeauthoryear{Culberson and
  Schaeffer}{1998}]{culberson1998pattern}
Joseph~C. Culberson and Jonathan Schaeffer.
\newblock Pattern databases.
\newblock {\em Computational Intelligence}, 14(3):318--334, 1998.

\bibitem[\protect\citeauthoryear{Culberson}{1999}]{Culberson1999}
Joseph~C. Culberson.
\newblock {S}okoban is {PSPACE}-{C}omplete.
\newblock In {\em Fun With Algorithms}, pages 65--76, 1999.

\bibitem[\protect\citeauthoryear{Duchi \bgroup \em et al.\egroup
  }{2011}]{duchi2011adaptive}
John Duchi, Elad Hazan, and Yoram Singer.
\newblock Adaptive subgradient methods for online learning and stochastic
  optimization.
\newblock {\em Journal of Machine Learning Research}, 12(61):2121--2159, 2011.

\bibitem[\protect\citeauthoryear{Ebendt and
  Drechsler}{2009}]{ebendt2009weighted}
R{\"u}diger Ebendt and Rolf Drechsler.
\newblock Weighted a* search – unifying view and application.
\newblock {\em Artificial Intelligence}, 173(14):1310 -- 1342, 2009.

\bibitem[\protect\citeauthoryear{Frank and Wolfe}{1956}]{frank1956quadratic}
Marguerite Frank and Philip Wolfe.
\newblock An algorithm for quadratic programming.
\newblock {\em Naval Research Logistics Quarterly}, 3(1-2):95--110, 1956.

\bibitem[\protect\citeauthoryear{Guez \bgroup \em et al.\egroup
  }{2018}]{boxobanlevels}
Arthur Guez, Mehdi Mirza, Karol Gregor, Rishabh Kabra, Sebastien Racaniere,
  Theophane Weber, David Raposo, Adam Santoro, Laurent Orseau, Tom Eccles, Greg
  Wayne, David Silver, Timothy Lillicrap, and Victor Valdes.
\newblock An investigation of model-free planning: boxoban levels.
\newblock \url{https://github.com/deepmind/boxoban-levels/}, 2018.
\newblock Accessed: 2023-05-01.

\bibitem[\protect\citeauthoryear{Guez \bgroup \em et al.\egroup
  }{2019}]{guez2019planning}
Arthur Guez, Mehdi Mirza, Karol Gregor, Rishabh Kabra, Sebastien Racaniere,
  Theophane Weber, David Raposo, Adam Santoro, Laurent Orseau, Tom Eccles, Greg
  Wayne, David Silver, and Timothy Lillicrap.
\newblock An investigation of model-free planning.
\newblock In {\em Proceedings of the 36th International Conference on Machine
  Learning}, volume~97 of {\em Proceedings of Machine Learning Research}, pages
  2464--2473. PMLR, 2019.

\bibitem[\protect\citeauthoryear{Hazan}{2016}]{hazan2016oco}
Elad Hazan.
\newblock Introduction to online convex optimization.
\newblock {\em Foundations and Trends{\textregistered} in Optimization},
  2(3-4):157--325, 2016.

\bibitem[\protect\citeauthoryear{Hinton}{2002}]{hinton2002poe}
Geoffrey~E. Hinton.
\newblock {Training Products of Experts by Minimizing Contrastive Divergence}.
\newblock {\em Neural Computation}, 14(8):1771--1800, 08 2002.

\bibitem[\protect\citeauthoryear{{Jabbari Arfaee} \bgroup \em et al.\egroup
  }{2011}]{ArfaeeZH11}
S.~{Jabbari Arfaee}, S.~Zilles, and R.~C. Holte.
\newblock Learning heuristic functions for large state spaces.
\newblock {\em Artificial Intelligence}, 175(16-17):2075--2098, 2011.

\bibitem[\protect\citeauthoryear{Jaggi}{2013}]{jaggi2013duality}
Martin Jaggi.
\newblock Revisiting {Frank-Wolfe}: Projection-free sparse convex optimization.
\newblock In {\em Proceedings of the 30th International Conference on Machine
  Learning}, volume 28(1) of {\em Proceedings of Machine Learning Research},
  pages 427--435, Atlanta, Georgia, USA, 17--19 Jun 2013. PMLR.

\bibitem[\protect\citeauthoryear{Joulani \bgroup \em et al.\egroup
  }{2013}]{joulani2013delayed}
Pooria Joulani, Andras Gyorgy, and Csaba Szepesvari.
\newblock Online learning under delayed feedback.
\newblock In {\em Proceedings of the 30th International Conference on Machine
  Learning}, volume 28(3) of {\em Proceedings of Machine Learning Research},
  pages 1453--1461. PMLR, 2013.

\bibitem[\protect\citeauthoryear{Korf}{1997}]{korf1997cube}
Richard~E. Korf.
\newblock Finding optimal solutions to rubik's cube using pattern databases.
\newblock In {\em Proceedings of the Fourteenth National Conference on
  Artificial Intelligence and Ninth Conference on Innovative Applications of
  Artificial Intelligence}, AAAI'97/IAAI'97, page 700–705. AAAI Press, 1997.

\bibitem[\protect\citeauthoryear{Mattern}{2013}]{mattern2013geomix}
Christopher Mattern.
\newblock Linear and geometric mixtures - analysis.
\newblock {\em Proceedings of the Data Compression Conference}, pages 301--310,
  02 2013.

\bibitem[\protect\citeauthoryear{Mattern}{2016}]{mattern2016phd}
Christopher Mattern.
\newblock {\em On Statistical Data Compression}.
\newblock PhD thesis, Technische Universit\"at Ilmenau, Fakult\"at f\"ur
  Informatik und Automatisierung, Feb 2016.

\bibitem[\protect\citeauthoryear{Matthew}{2005}]{matthew2005adaptive}
V~Mahoney Matthew.
\newblock Adaptive weighing of context models for lossless data compression.
\newblock {\em Florida Institute of Technology CS Dept, Technical Report
  CS-2005-16, https://www. cs. fit. edu/Projects/tech\_reports/cs-2005-16.
  pdf}, 2005.

\bibitem[\protect\citeauthoryear{Mittal \bgroup \em et al.\egroup
  }{2022}]{mittal2022expose}
Dixant Mittal, Siddharth Aravindan, and Wee~Sun Lee.
\newblock Expose: Combining state-based exploration with gradient-based online
  search, 2022.
\newblock arXiv 2202.01461.

\bibitem[\protect\citeauthoryear{Nesterov}{1983}]{nesterov1983agd}
Yurii Nesterov.
\newblock A method for solving the convex programming problem with convergence
  rate $o(1/k^2)$.
\newblock {\em Proceedings of the USSR Academy of Sciences}, 269:543--547,
  1983.

\bibitem[\protect\citeauthoryear{Orabona and Pál}{2018}]{orabona2018solo}
Francesco Orabona and Dávid Pál.
\newblock Scale-free online learning.
\newblock {\em Theoretical Computer Science}, 716:50--69, 2018.

\bibitem[\protect\citeauthoryear{Orseau and Hutter}{2021}]{orseau2021isotuning}
Laurent Orseau and Marcus Hutter.
\newblock Isotuning with applications to scale-free online learning, 2021.

\bibitem[\protect\citeauthoryear{Orseau and Lelis}{2021}]{orseau2021policy}
Laurent Orseau and Levi H.~S. Lelis.
\newblock Policy-guided heuristic search with guarantees.
\newblock {\em Proceedings of the AAAI Conference on Artificial Intelligence},
  35(14):12382--12390, May 2021.

\bibitem[\protect\citeauthoryear{Orseau \bgroup \em et al.\egroup
  }{2018}]{orseau2018single}
Laurent Orseau, Levi Lelis, Tor Lattimore, and Theophane Weber.
\newblock Single-agent policy tree search with guarantees.
\newblock In S.~Bengio, H.~Wallach, H.~Larochelle, K.~Grauman, N.~Cesa-Bianchi,
  and R.~Garnett, editors, {\em Advances in Neural Information Processing
  Systems}, volume~31. Curran Associates, Inc., 2018.

\bibitem[\protect\citeauthoryear{Pearl}{1984}]{pearl1984heuristics}
Judea Pearl.
\newblock {\em Heuristics: Intelligent Search Strategies for Computer Problem
  Solving}.
\newblock Addison-Wesley Longman Publishing Co., Inc., USA, 1984.

\bibitem[\protect\citeauthoryear{Pohl}{1970}]{pohl1970heuristic}
Ira Pohl.
\newblock Heuristic search viewed as path finding in a graph.
\newblock {\em Artificial Intelligence}, 1(3):193 -- 204, 1970.

\bibitem[\protect\citeauthoryear{Ratner and Warmuth}{1986}]{ratner1986puzzle}
Daniel Ratner and Manfred Warmuth.
\newblock Finding a shortest solution for the nxn extension of the 15-puzzle is
  intractable.
\newblock In {\em Proceedings of the Fifth AAAI National Conference on
  Artificial Intelligence}, AAAI'86, page 168–172. AAAI Press, 1986.

\bibitem[\protect\citeauthoryear{Rissanen}{1983}]{rissanen1983universal}
J.~Rissanen.
\newblock A universal data compression system.
\newblock {\em IEEE Transactions on Information Theory}, 29(5):656--664, 1983.

\bibitem[\protect\citeauthoryear{Shalev-{S}hwartz}{2007}]{shalev2007online}
Shai Shalev-{S}hwartz.
\newblock {\em Online learning: Theory, algorithms, and applications}.
\newblock PhD thesis, Hebrew University, Jerusalem, 2007.

\bibitem[\protect\citeauthoryear{Sutton and
  Barto}{1998}]{sutton1998reinforcement}
Richard~S. Sutton and Andrew~G. Barto.
\newblock {\em Reinforcement Learning : An Introduction}.
\newblock MIT Press, 1998.

\bibitem[\protect\citeauthoryear{Truong and
  Nguyen}{2021}]{truong2021backtracking}
Tuyen Truong and Hang-Tuan Nguyen.
\newblock Backtracking gradient descent method and some applications in large
  scale optimisation. part 2: Algorithms and experiments.
\newblock {\em Applied Mathematics \& Optimization}, 84:1--30, 12 2021.

\bibitem[\protect\citeauthoryear{Veness \bgroup \em et al.\egroup
  }{2017}]{veness2017gln}
Joel Veness, Tor Lattimore, Avishkar Bhoopchand, Agnieszka Grabska-Barwinska,
  Christopher Mattern, and Peter Toth.
\newblock Online learning with gated linear networks, 2017.
\newblock arXiv 1712.01897.

\bibitem[\protect\citeauthoryear{Veness \bgroup \em et al.\egroup
  }{2021}]{veness2021gln}
Joel Veness, Tor Lattimore, David Budden, Avishkar Bhoopchand, Christopher
  Mattern, Agnieszka Grabska-Barwinska, Eren Sezener, Jianan Wang, Peter Toth,
  Simon Schmitt, and Marcus Hutter.
\newblock Gated linear networks.
\newblock {\em Proceedings of the AAAI Conference on Artificial Intelligence},
  35(11):10015--10023, May 2021.

\bibitem[\protect\citeauthoryear{Willems \bgroup \em et al.\egroup
  }{1995}]{willems1995ctw}
Frans M.~J. Willems, Yuri~M. Shtarkov, and Tjalling~J. Tjalkens.
\newblock The context tree weighting method: Basic properties.
\newblock {\em IEEE Transactions on Information Theory}, 41:653--664, 1995.

\bibitem[\protect\citeauthoryear{Zinkevich}{2003}]{Zin03}
Martin Zinkevich.
\newblock Online convex programming and generalized infinitesimal gradient
  ascent.
\newblock In {\em Proceedings of the Twentieth International Conference on
  Machine Learning}, pages 928--935, 2003.

\end{thebibliography}

\ifarxiv
\clearpage

\section{Bootstrap Algorithm Details}\label{apdx:bootstrap}

\begin{algorithm}
\caption{Bootstrap using \code{LTS_CM} (given in \cref{alg:ltscm}),
which returns \code{"budget_reached"} when the number of nodes expanded reaches $B_t$,
or returns a solution node $n^*\in \nodeset^*(n^0)$
if it reaches $n^*$,
or returns \code{"no_solution"} if all nodes have been expanded without exhausting the budget and without reaching a solution node, which means that the problem has no solution.
}
\label{alg:bootstrap}
\begin{lstlisting}
# $\nodeset^0$: set of root nodes = problems
# $B_1$: initial budget
# $\pol_1$: initial policy
# Returns the set of solution nodes
def Bootstrap_with_LTS($\nodeset^0$, $B_1$, $\pol_1$):
  solns = {} # dictionary of problem -> solution
  for t = 1, 2, ...:
    for each $n^0 \in \nodeset^0$:
      result = LTS_CM($n^0$, $B_t$, $\beta_t$) # search
      if result is "no_solution": $\nodeset^0 \leftarrow \nodeset^0 \setminus\{n^0\}$
      if result is a node $n^*$: soln[$n^0$] = $n^*$
    if len(soln) = $|\nodeset^0|$: return soln
    # Update the parameters of the model
    $\beta_{t+1} \approx \argmin_{\beta\in\betaset} \Loss($soln.values()$,\beta) + R(\beta)$
    choose budget $B_{t+1}$
\end{lstlisting}
\end{algorithm}

\citet{orseau2021policy}
use a variant of the Bootstrap process~\cite{ArfaeeZH11}
to iteratively solve a set of problems while improving the policy based on the solutions for the already solved problems.
See \cref{alg:bootstrap}.

At each Bootstrap iteration $t$, LTS is run on each problem (even those already solved) with a budget of $B_t$ node expansions with the context-model policy with current parameters $\beta^t$.
After collecting the set of solutions $\nodeset_t$,
the parameters $\beta^{t+1}$ are obtained from \cref{eq:beta_full_update} to some approximation (see \cref{apdx:convex_optim}),
and the next bootstrap iteration $t+1$ is started with budget $B_{t+1}$.

Adjusting the budget is non trivial.
Keep in mind that computation time during search is proportional to the number of expansions.
Solving previously solved problems usually is fast, because the policy has been optimized for them.
Each problem for which a solution is newly found usually takes a large fraction of the budget
(since they couldn't be solved for the previous budget),
and every problem that remains unsolved consumes the whole budget.
While a larger budget means that more problems can be solved, for a fixed set of parameters it is usual to see this number grow only logarithmically with the budget
(since we are tackling hard problems).
Hence when only few problems have already been solved, a large budget will make the algorithm spend a lot of time in yet-unsolvable problems, wasting computation time.
By contrast, a too small budget will prevent finding new solutions and improving the policy, requiring more Bootstrap iterations.

\citet{ArfaeeZH11} double the budget at each new Bootstrap iteration. This can become wasteful in computation time if learning works well, but a larger budget does not help much, in which case it may be better to use a constant budget.
It may also be not fast enough during the last iterations: suppose 95\% of the problems are solved, but the remaining 5\% ones require to double the budget $k$ more times:
then the 95\% will be resolved $k$ more times (possibly finding different solutions), and, if the found solutions change, optimization is also performed $k$ times before any new problem can be solved.
By contrast, \citet{orseau2021policy} use a fixed budget and double the budget
only if no new problem is solved, which can also be wasteful in computation time if learning does not manage to work well enough and just one more problem is solved at each step.

To alleviate these issues, first, if more than a factor $(1+b)$ (say $b=1/4$)
of problems are solved at iteration $t$ compared to the previous iteration
--- formally, $|\nodeset_{t}| \geq (1+b)\left|\bigcup_{t'<t}\nodeset_{t'}\right|$ ---
we reduce the next budget in case the current budget is too high:
\begin{align*}
    B_{t+1} = \max\{B_1, B_{t}/2\}\,.
\end{align*}
Otherwise, we increase $B_{t+1}$ so as to approximately double the number of total expansions (in the worst case of no new problem solved), rather than merely doubling the budget.
More precisely,
at Bootstrap iteration $t$,
say the total number of expansions of \emph{solved} problems is $T^+_t$, 
and the number of remaining unsolved problems is $s^-_t = |\nodeset^* \setminus\bigcup_{t'\leq t}\nodeset_{t'}|$,
then we set the next budget 
\begin{align*}
    B_{t+1} = 2B_t + T^+_t/s^-_t\,,
\end{align*}
which ensures that $T^+_t + s^-_tB_{t+1} = 2(T^+_t + s^-_tB_t)$,
where $T^+_t + s^-_tB_t$ is the actual number of expansions used during iteration $t$,
and $T^+_t + s^-_tB_{t+1}$ is the probable number of expansions in case no new problem is solved, and assuming that previous problems take about the same time to be solved again 
(which is likely to be an overestimate due to learning).

\section{Formal Statement of the Lower Bound}\label{apdx:lower_bound}

In this section we provide a formal version of the informal lower bound of \cref{thm:informal_lower_bound} on the number of node expansions required before reaching a target node $n^*$.
This number is within a factor $(A-1)\bar d$ of the upper bound, showing that the upper bound is quite tight and can be meaningfully used as a loss function.

The lower bound in \Cref{thm:lower_bound} below requires the following lemma,
which shows that the probability mass at the root behaves like a liter of water
that is distributed recursively (but unevenly) along all the branches, and that if we collect the water at all the leaves (assuming a finite tree) then it still amounts to one liter, as long as the policy is proper.
This lemma can be found in a compact form in the proof of Theorem 3~\cite{orseau2018single}.

A tree $\nodeset'\subseteq\nodeset$ is said to be \emph{full}
if every node of the tree either has all its children in the tree, or none of them.

\begin{lemma}\label{lem:sum_to_one}
Let $\nodeset'\subseteq\nodeset$ be a finite tree with root $n_0$, and let $\leafset'\subseteq\nodeset'$ be its leaves.
Let $\pol$ be a policy with $\pol(n_0)=1$.
Then $\sum_{n\in\leafset'} \pol(n) \leq 1$.
Furthermore, if the policy is proper and $\nodeset'$ is a full tree, then 
$\sum_{n\in\leafset'} \pol(n) = 1$.
\end{lemma}
\begin{proof}
We start with the equality case.
Using the fact that the policy is proper on the second line,
and the fact that the tree $\nodeset'$ is full on the fourth line,
we have
\begin{align*}
    \sum_{n\in\leafset'}\pol(n) &= 
    \sum_{n\in\nodeset'}\pol(n) - \sum_{n\in\nodeset'\setminus\leafset'} \pol(n) \\
    &=\sum_{n\in\nodeset'}\pol(n) - \sum_{n\in\nodeset'\setminus\leafset'} \pol(n)\sum_{n'\in\children(n)}\pol(n'\mid n) \\
    &=\sum_{n\in\nodeset'}\pol(n) - \sum_{n\in\nodeset'\setminus\leafset'} \sum_{n'\in\children(n)}\pol(n') \\
    &=\sum_{n\in\nodeset'}\pol(n) - \sum_{n\in\nodeset'\setminus\{n_0\}}\pol(n') \\
    &= \pol(n_0)=1\,.
\end{align*}
If the tree $\nodeset'$ is not full, it suffices to assign probability 0 to children outside of $\nodeset'$, which reduces to an improper policy.
If the policy is not proper, it can be made proper on $\nodeset'$ by renormalization of $\pol$ to $\tilde\pol$.
More precisely, if the tree $\nodeset'$ is not full or the policy is not proper,
define,
for all $n\in\nodeset'\setminus\leafset'$,
for all $n'\in\children(n)\cap \nodeset': \tilde\pol(n'\mid n) = \pol(n'\mid n) / \sum_{n''\in\children(n)\cap \nodeset'} \pol(n''\mid n)$,
which ensures that $\tilde\pol(n) \geq \pol(n)$ for all nodes $n\in\nodeset'$, 
and thus $\sum_{n\in\leafset'} \pol(n) \leq \sum_{n\in\leafset'} \tilde\pol(n) =1$.
\end{proof}

For a node $n^*$,
define $\nodesetcost(n^*) = \{n\in\nodeset: \rootop(n)=\rootop(n^*)\land
\dop(n) \leq \dop(n^*)\}$,
which is the set of nodes of the same tree of cost at most that of $n^*$,
and $\leafset'(n^*) = \{n\in\nodeset: \dop(n) > \dop(n^*) \land \parent(n) \in \nodesetcost(n^*)\}$, which is the set of children right outside $\nodesetcost(n^*)$
--- $\leafset'(n^*)$ would be the `frontier' or the contents of the priority queue in \cref{alg:ltscm}, disregarding tie breaking.
Let $A\geq 2$ be the maximal branching factor of the search tree $\nodesetcost(n^*)$,
that is, for all $n\in\nodeset: |\children(n)| \leq A$.
Observe that $\nodesetcost(n^*)$ may not be a full tree,
but that $\nodesetcost(n^*) \cup \leafset'(n^*)$ is a full tree.

\begin{theorem}[Lower bound]\label{thm:lower_bound}
Let $\pol$ be a proper policy.
Then, for a node $n^*$, 
the number of nodes with cost at most that of $n^*$ is at least
\begin{align*}
    |\nodesetcost(n^*)| \geq \frac{1}{(A-1)}\left(\frac{1}{\bar d}\dop(n^*)-1\right)\,.
\end{align*}
where $\bar d =\smash{1/\sum_{n\in\leafset'(n^*)}\frac{\pol(n)}{d(n)}}$
is the harmonic average of the depth at the leaves $\leafset'(n^*)$.
\end{theorem}
Also observe that by the harmonic-mean -- arithmetic mean inequality, $\bar d \leq \sum_{n\in\leafset'(n^*)}\pol(n)d(n)$, the average depth at the leaves of the search tree $\nodesetcost(n^*)$.
\begin{proof}
First, using the fact that $|\leafset'(n^*)\cup\nodesetcost(n^*)|$ is a full tree,
\begin{multline*}
    |\leafset'(n^*)|+|\nodesetcost(n^*)| = |\leafset'(n^*)\cup\nodesetcost(n^*)| \\
    = 1+ \sum_{n\in\nodesetcost(n^*)}\sum_{n'\in\children(n)} 1
    \leq 1+ A|\nodesetcost(n^*)|\,,
\end{multline*}
and by rearranging we obtain
\begin{align*}
    |\nodesetcost(n^*)| &\geq (|\leafset'(n^*)|-1)/(A-1)\,.
\end{align*}
Now, 
\begin{align*}
    |\leafset'(n^*)| &= \sum_{n\in\leafset'(n^*)} \dop(n)\frac{\pol(n)}{d(n)} \\
    &> \dop(n^*) \sum_{n\in\leafset'(n^*)}\frac{\pol(n)}{d(n)}
    = \frac{1}{\bar d}\dop(n^*) \,,
\end{align*}
Since $\leafset'(n^*)$ are the leaves of a full tree,
$\sum_{n\in\leafset'(n^*)} \pol(n) =1$ by \cref{lem:sum_to_one}
and thus $\bar d$ is indeed an harmonic mean of the depths of the leaves.
Therefore,
\begin{align*}
    |\nodesetcost(n^*)| \geq \frac{1}{(A-1)}\left(\frac{1}{\bar d}\dop(n^*)-1\right)\,.
    &\qedhere
\end{align*}
\end{proof}

\begin{remark}
It appears that the factor $1/(A-1)$ is necessary.
Consider the tree in \cref{fig:bad_lower_bound_tree}, and take $A \geq 3$.
Then $\dop(n_{2,1})=2A$ while
$\dop(n_{1,\cdot})=2A^2/(A-2) > \dop(n_{2, 1})$.
Then $|\nodesetcost(n_{2, 1})| = 5= 5\dop(n_{2, 1})/(2A) = O(\dop(n_{2,1})/(A-1)).$
Also note that replacing $2/A$ with $1/A$ in the tree
leads to $\dop(n_{2, 1})=4A$ and 
$\dop(n_{1,\cdot})=2A^2/(A-1) < \dop(n_{2, 1})$,
which means that $|\nodesetcost(n_{2, 1})|\geq 5+A=\Omega(\dop(n_{2, 1}))$ instead.
\end{remark}

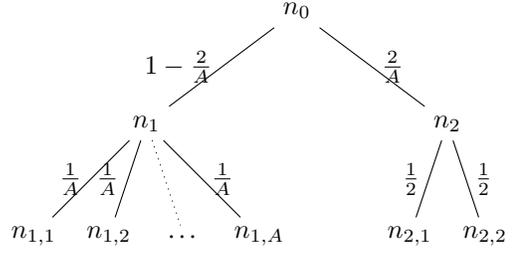
\begin{figure}
\centering
\begin{tikzpicture}
\tikzstyle{level 1}=[sibling distance=40mm]
\tikzstyle{level 2}=[sibling distance=10mm]
\tikzstyle{level 3}=[sibling distance=10mm]
\node (root) {$n_0$}
    child {
        node (left1) {$n_{1}$}
        child {
            node (left2-1) {$n_{1,1}$}
            edge from parent node[left, draw=none] {$\frac{1}{A}$}
        }
        child {
            node (left2-2) {$n_{1,2}$}
            edge from parent node[left, draw=none] {$\frac{1}{A}$}
        }
        child [dotted] {
            node (left2-3) {\ldots}
        }
        child {
            node (left2-10) {$n_{1,A}$}
            edge from parent node[right, draw=none] {$\frac{1}{A}$}
        }
        edge from parent node[left] {$1-\frac2A$}
    }
    child {
        node (right1) {$n_{2}$}
        child {
            node (right2-1) {$n_{2,1}$}
            edge from parent node[left, draw=none] {$\frac12$}
        }
        child {
            node (right2-2) {$n_{2,2}$}
            edge from parent node[right, draw=none] {$\frac12$}
        }
        edge from parent node[right] {$\frac2A$}
    }
;
\end{tikzpicture}
\caption{A tree showing the necessity of the factor $A-1$ for the lower bound,
 with $A \geq 3$.
Nodes that have fewer than $A$ children can be completed with children of probability 0.}
\label{fig:bad_lower_bound_tree}
\end{figure}

\section{Convex optimization algorithm}\label{apdx:convex_optim}

To minimize \cref{eq:ltsloss_cm}, many convex optimization algorithms can be considered.
For a first simple implementation, we would recommend using Frank-Wolfe
\cite{frank1956quadratic,jaggi2013duality} while being mindful of numerical stability 
(see \cref{apdx:num_stable}).

In the following, we describe the convex optimization routine we use for the experiments, however we suspect better and possibly more principled algorithms might be applicable too.

For the optimizer, we use isoGD~\cite{orseau2021isotuning} with projection onto $\betaset$, a scale-free variant of Adagrad~\cite{duchi2011adaptive} (see also SOLO-FTRL~\cite{orabona2018solo}), which is adapted to use a line search.
We have observed empirically that these algorithms tend to close the duality gap~\cite{jaggi2013duality} faster than other algorithms --- such as Frank-Wolfe~\cite{frank1956quadratic}, normalized gradient descent~\cite{cortes2006ngd}, accelerated gradient descent~\cite{nesterov1983agd}), but admittedly we did not try all variants of all algorithms.
However, the well-known difficulty with Adagrad-style algorithms is that the learning rate is always smaller than $1/G$ where $G$ is the magnitude of largest observed gradient.
But in function optimization, almost always the largest gradients are observed early and decrease significantly near the optimum --- in our case, the initial gradients can be exponentially large.
Hence, to reduce this dependency, we reset the learning rates on steps that are powers of 2 --- this is known as the doubling trick~\cite{cesabianchi2006prediction} and we fully expect that a regret bound can be proven with resets too, paying only a constant factor in the regret bound for a significantly milder dependency on the usually-larger initial gradients.
We use one learning rate per context.

Optimization is stopped after 200 iterations: if this happens to not be enough to improve the policy significantly to solve more problems, 200 more iterations will be triggered anyway after the next Bootstrap iteration.

Optimization is also stopped early if the duality gap~\cite{jaggi2013duality} 
guarantees that the loss is within a factor 2 of the optimum. Recall that
this roughly means that the bound on the search time is a factor 2 away
from the bound on the search time for the \emph{optimal} parameters.
The duality gap is calculated every 20 iterations to amortize the computation cost.
See also \cref{apdx:betasimplex}.

For the line search, we use some ideas from \citet{truong2021backtracking}:
A line search is triggered at each update iteration $t$ where $1\leq t\mod 20\leq 3$, and the learning rate found by the line search is re-used for the next optimization steps as long as there is improvement --- otherwise a line search is triggered too ---  and also as the first middle query of the line search in [0, 1].
We use a (quasi-)exact line search rather than a backtracking line search,
as it can make a significant difference on the first iterations, but less so afterwards.

See also \cref{apdx:num_stable} on numerical stability.

\paragraph{Training times.}
See the main text \cref{sec:Exp} for the description of the hardware.
The total training time for The Witness is 25min, for Boxoban it is 1h02, for STP it is 1h03, and for Boxoban hard it is 11h15.
See \cref{apdx:results} for Rubik's Cube.

\section{Numerical Stability Considerations}\label{apdx:num_stable}

\newcommand{\lse}{\operatorname{LSE}}

One should use a numerically stable `softmax' function for product mixing,
and a stable `log-sum-exp' (LSE) to calculate the logarithm of the LTS loss --- the LTS loss
can be exponentially large for untuned parameters.
\begin{align*}
    \lse(X) &= C + \log \sum_{x\in X} \exp(x - C)\,,&  C&= \max X\,.
\end{align*}
For a set $\nodeset'$ of nodes,
We can rewrite \cref{eq:ltsloss_cm,eq:ltsinstantloss_cm}, and the scaled gradient  as:
\begin{align*}
    &\log \loss(n, \beta) = 
    \log d(n) 
    - \sum_{j=0}^{d(n)-1} \log \prodmix(n_{[j]}, a(n_{[j+1]}), \beta)\,,\\
    &\log \Loss(\nodeset',\beta) = 
    \lse(\{\log\loss(n,\beta) \mid n\in\nodeset'\})
    \\
    &\log(\Loss(\nodeset',\beta) + R(\beta))
    = \lse(\log \Loss(\nodeset',\beta), \log R(\beta))\,.
\end{align*}
Recall that $\epsmix=0$ during optimization.
During the line search, one can use
$\log(\Loss(\nodeset',\beta) + R(\beta))$ 
but note that, while still unimodal (quasiconvex), it may not be convex anymore  with a quadratic regularizer (despite \cref{thm:logL_convex} below).

To calculate the gradients, similar caution should be used, for example
for some constant $C$:
\begin{align*}
&\beta^{t+1} = \argmin_{\beta}  e^{-C} R(\beta) + \sum_{\tau\in\mathcal{T}} \exp(\log \loss(\tau, \beta) - C)\,, \\
    &\nabla\exp(\log \loss(\tau, \beta) - C)
    = \exp(\log \loss(\tau, \beta) - C)\nabla \log \loss(\tau, \beta)\,. \\
\end{align*}

\section{Convexity}

\subsection{Log loss to LTS loss}\label{apdx:logloss_to_ltsloss}

\begin{proof}[Proof of \cref{thm:logloss_to_inverseloss}]
We can write
\begin{align*}
    L(x) = \sum_k \exp\left(\sum_t-\log f_{k,t}(x)\right)\,,
\end{align*}
By assumption, $-\log f_{k,t}(x)$ is convex for all $k,t$,
and convexity is preserved by both summation and exponentiation
(convex and non-decreasing)~\cite{boyd2004convex},
hence $L(x)$ is convex.
\end{proof}

\begin{remark}
Convexity in LTS loss does not imply convexity in log loss.
For example, take $f(x) = 1/x^2$ for $x>0$, then $1/f$ is (strongly) convex,
but $-\log 1/x^2 = 2\log |x|$ is not only concave on $(-\infty,0)$ and on $(0,\infty)$ but also has a singularity at 0.
In a sense, the LTS loss is `nicer' for convex optimization than the log loss.
\end{remark}

\begin{theorem}\label{thm:logL_convex}
The logarithm of loss function $L(\cdot)$ defined in \cref{thm:logloss_to_inverseloss}
is convex.
\end{theorem}
\begin{proof}
Following the proof of \cref{thm:logloss_to_inverseloss} we can write
\begin{align*}
    \log L(x) = \log \sum_k \exp\left(\sum_t-\log f_{k,t}(x)\right)\,,
\end{align*}
and the result follows by observing that log-sum-exp and summation preserve convexity~\cite{boyd2004convex}.
\end{proof}

\subsection{LTS loss convexity of product mixing of context predictors}\label{sec:LTSconv}

\begin{theorem}
The function $\Loss(\nodeset_t, \beta)$ defined in \cref{eq:ltsloss_cm} is convex in $\beta$.
\end{theorem}
\begin{proof}
This function is of the form $\sum \exp \sum \log \sum \exp f(\beta)$
where $f$ is linear.
The result follows by observing that summation, exponentiation, and log-sum-exp are all preserving convexity, since they all are convex and non-decreasing~\cite{boyd2004convex}.
\end{proof}

We now provide a more general result that applies if the context predictors
are members of the exponential family, rather than just categorical distributions.

\begin{lemma}
Let $\actionset$ be a finite set of actions and canonical parameters $\beta\in\Reals^{\contextset\times\actionset}$, and let $\contextset$ be a set of predictors.
Let 
\begin{align*}
    p_c(n, a; \beta) = \exp[\beta \cdot T_c(n, a) - A_c(n,\beta) + B_c(n, a)]
\end{align*} 
and all node $n\in\nodeset$ and all actions  $a\in \actionset$,
with $A_c:\nodeset\times\betaset\to\Reals$ and $B_c:\nodeset\times\actionset\to\Reals$ 
and $T_c:\nodeset\times\actionset\to\Reals^{\contextset\times\actionset}$
for all predictors $c\in\contextset$,
be a set of members of the exponential family in canonical form,
with a dependency on the current node $n\in\nodeset$.
Then the product mixing $\prodmix(n, a;\beta)$ (\cref{eq:prodmix})
of the $\{p_c\}_{c\in\contextset}$ is also a member of the exponential family in canonical form:
\begin{align*}
    \prodmix(n, a;\beta) &=\exp[\beta\cdot T(n, a) - A(n, \beta) + B(n, a)]\,,\\
\intertext{with
}
    T(n, a) &= \sum_{c\in\contextset} T_c(n,a)\,, \\
    B(n, a) &= \sum_{c\in\contextset} B_c(n,a)\,, \\
    A(n, \beta) &= \ln \sum_{a'\in \actionset}
    \exp\left[\beta\cdot T(n,a') + B(n,a') \right]\,. 
\end{align*}
\end{lemma}
\begin{proof}
The result is a straightforward application of the definition of product mixing:
\begin{align*}
    \prodmix&(n, a;\beta) \\
    &= \frac{\prod_{c\in\contextset}p_c(n,a;\beta)}{\sum_{a'\in \actionset}\prod_{c\in\contextset}p_c(n,a';\beta)}\\
    &= \frac{\exp\left[\beta\cdot \sum_{c\in\contextset}T_c(n,a)  + \sum_{c\in\contextset}B_c(n,a) \right]}{
    \sum_{a'\in \actionset}
    \exp\left[\beta\cdot \sum_{c\in\contextset}T_c(n,a') + \sum_{c\in\contextset}B_c(n,a') \right]} \\
    &= 
    \exp[\beta\cdot T(n,a) - A(n, \beta) + B(n,a)]\,.
    \qedhere
\end{align*}
\end{proof}
Categorical context models can be expressed as members of the exponential family in canonical form, by setting $T_c(n,a)$ to a zero vector, with just a single 1
at index $(c, a)$ for context $c$ and action $a$, but only if the context $c$ is active at node $n$, \ie $c\in\contextset(n)$, that is
\begin{align*}
    T_c(n,a)_{c',a'} ~=~ \indicator{c'=c}\cdot\indicator{a'=a}\cdot\indicator{c\in\contextset(n)},
\end{align*}
where $\indicator{test}=1$ if $test$ is true, 0 otherwise.
This implies that the vector $T(n,a)$ also has a 1 at index $(c, a)$ for each active context $c$, for each action $a\in\actionset$.
To select only the valid actions at node $n$,
also set $B_c(n, a)=0$ if $a\in\actionset(n)$ and $B_c(n,a)=-\infty$ otherwise.
Then
\begin{align*}
    p_c(n, a ; \beta) &= \exp(\beta \cdot T_c(n, a)  - A_c(n, \beta) + B_c(n, a))\\
    &= \exp(\beta_{c,a}\indicator{c\in\contextset(n)} - A_c(n,\beta))\indicator{a\in\actionset(n)}\,, \\
    A_c(n, \beta) &= \ln \sum_{a\in\actionset(n)} \exp(\beta_{c, a}\indicator{c\in\contextset(n)})\,,\\
    \intertext{that is,}
    p_c(n, a ; \beta) &=\begin{cases}
    \frac{\exp \beta_{c, a}}{\sum_{a'\in\actionset(n)}\exp \beta_{c,a'}} &\text{if } c\in\contextset(n), a\in\actionset(n)\,, \\
    0 &\text{if } a\notin\actionset(n)\,,\\
    \frac{1}{|\actionset(n)|} &\text{otherwise.}
    \end{cases}
\end{align*}
Recall that uniform distributions have no effects in the product mixing,
and thus a context that is not active is in effect removed from the product mixing,
as in \cref{eq:prodmix_softmax}.

It is interesting to note that the $A_c(n, \beta)$ do not appear directly
in the resulting form of the product mixing and thus do not need to be calculated.

Beyond simple context models, since the vector $T_c(n, a)$ can depend on the current node $n$,
it can make use in particular of \emph{features} of the corresponding state of the environment, such as a heuristic distance to the goal.

Finally, members of the exponential family in canonical form are well-known to be 
log-concave in their natural parameters $\beta$,
that is, 
their log loss is convex in their natural parameters:
$-\log p_c(\cdot,\cdot;\beta)$ is of the form 
$h(\beta) + \ln \sum \exp g(\beta)$ where $h$ and $g$ are linear,
and since log-sum-exp is convex the result follows.
Therefore, by \cref{thm:logloss_to_inverseloss},
the LTS loss of the product mixing of members of the canonical exponential family 
is convex in their natural parameters.

\section{Beta-Simplex}\label{apdx:betasimplex}

This section describes a small but computationally helpful improvement regarding the calculation of the duality gap~\cite{jaggi2013duality}, which is used to terminate the optimization procedure.

The domain of the parameters $\beta$ is defined in the main paper as 
$\betaset=[\ln\epslow, 0]^{|\contextset|\times A}$,
and wrote that $p_c(a;\beta) \geq \epslow/|\children(n)|$.

While the duality gap can be calculated on this set,
it can also be calculated for a subset of $\betaset$,
which more closely relates to the probability distributions of the predictors.
Furthermore, the regret can still be meaningfully compared to 
the best probability distributions for the context predictors,
rather than the optimal parameters $\beta^*_t\in\betaset$.

The highest-entropy probability distribution is the uniform distribution
and can be  expressed with $p_c$ by setting all components $\beta_{c, \cdot}$ to the same value.

The lowest-entropy probability distribution that can be expressed with $p_c$ is
such that $\beta_{c, a}= 0$ for some chosen $a\in\actionset$ and $\beta_{c, a'}=\ln\epslow$ for $a'\neq a$,
giving 
\begin{align*}
    p_c(a; \beta) &= 1/(1+(A-1)\epslow)\geq 1-(A-1)\epslow\,,\\
    p_c(a';\beta) &= \epslow/(1+(A-1)\epslow)\leq \epslow\,.
\end{align*}

Consider the constrained simplex $\Delta_{\epslow}$ such that if $p\in\Delta_{\epslow}$,
then for all $a\in\actionset:p(a) \geq \epslow$.
Hence the probability distributions expressed by $\betaset$ 
can express at least as much as $\Delta_{\epslow}$ by convex combinations.
Unfortunately, enforcing $\sum_{a} \exp \beta_{\cdot,a} =1$ on $\betaset$ does not 
lead to a convex set.

Instead, we define the $\beta$-simplex as a (convex) subset of $\betaset$
by constraining $\sum_{a\in\actionset}\beta_{\cdot, a} = (A-1)\ln\epslow$.
Note that the $\beta$-simplex still contains the highest and lowest entropy distributions of $\Delta_{\epslow}$.
The `center' of the $\beta$-simplex is at $\beta_{.,a}=((A-1)\ln\epslow) / A$,
which we define as $\beta_0$,
and explains why we use the regularization $\|\beta-\beta_0\|^2$.

Hence, instead of calculating the regret compared to $\beta^* \in\betaset$,
we can also consider calculating the regret to the best distribution in $\Delta_{\epslow}$
or the best point in the $\beta$-simplex.

More importantly, we calculate the duality gap~\cite{jaggi2013duality} for the $\beta$-simplex,
which experimentally is easier to reduce
than the duality gap for the $\beta$-hypercube $\betaset$.

\section{Mutex Sets: Relative Tilings}\label{apdx:relative_tilings}

We now give a general and formal definition of the rectangular tiling example in Figure \Cref{fig:tiling}.
It is closely related to tile coding~\cite{sutton1998reinforcement}.
On a grid of dimension $R\times C$,
relative to some position $(r_0, c_0)$ on the grid,
we call a \emph{relative tile} $T(s_r, s_c, d_r, d_c)$ a particular \mutexset{} with
row span $s_r$, column span $s_c$,
row offset $d_r$ and column offset $d_c$.
The ordered values of the grid rectangle between $(r_0+d_r, c_0+d_c)$ and $(r_0+d_r+s_r-1, c_0+d_c+s_c-1)$ 
identify a unique context within the relative tile.
A padding value 
can be chosen arbitrarily for coordinates that are outside the grid.

A \emph{relative tiling} $R_T(s_r, s_c, D_r, D_c)$
of row span $s_r$, column span $s_c$, row distance $D_r$ and column distance $D_c$, relative to some position $(r_0, c_0)$ is a set of relative tile \mutexsets{}
$\{T(s_r, s_c, d_r, d_c)\}_{d_r, d_c}$ for $d_r\in[-D_r,\dots, D_r-s_r+1]$ and $d_c\in[-D_c,\dots,D_c-s_c+1]$ --- this ensures that the last position of the last relative tile is at 
$(r_0+D_r, c_0+D_c)$, while the first position of the first relative tile is at $(r_0-D_r,c_0-D_c)$.
There are $(2D_r+2-s_r)(2D_c+2-s_c)$ \mutexsets{} in the relative tiling.

In particular the mutex sets used in the experiments for Sokoban, The Witness, and the Sliding-Tile Puzzle are as follows:
The position $(r_0, c_0)$ is taken to be the position of the agent: the avatar in Sokoban, the blank tile in the sliding tile puzzle, and the tip of the `snake' in The Witness.

We used the following relative tilings.
For Sokoban: $R_T(3,3,4,4)$, 
$R_T(2,4,2,3)$, $R_T(4,2,3,2)$
$R_T(2,2,2,2)$, 
$R_T(1,2,1,1)$, $R_T(2,1,1,1)$,
the number of mutex sets is $|\patternset|=125$,
and walls are used as padding value;
For The Witness: $R_T(3,3,4,4)$, $R_T(2,2,4,4)$, $R_T(2,1,1,1)$, $R_T(1,2,1,1)$,
$|\patternset|=125$,
with one additional padding color, and the goal location is not encoded.
For the STP: $R_T(2,2,3,3)$, $R_T(2,1,2,2)$, $R_T(1,2,2,2)$, $R_T(1,1,2,2)$,
$|\patternset|=102$, with an additional padding value.

For The Witness, our implementation uses two grids: one for the (fixed) colors
and one for the snake (the trajectory of the player).
We perform the same relativing tiling on each grid in parallel,
merging each pair of contexts into a single context.

\section{Extended Table of Results}\label{apdx:results}

While the message of the main paper is to compare LTS+CM with LTS+NN,
it is also interesting to compare LTS+CM with other algorithms
for the same domains.
See \cref{tab:fullresults}.

In particular, \citet{orseau2021policy} introduce the PHS* algorithm,
which is based on LTS, but also uses a heuristic function to speed up the search,
and they also compare with Weighted A* (WA*)~\cite{pohl1970heuristic,ebendt2009weighted} which uses only a heuristic function, both with similar neural networks as LTS+NN.
We can see LTS+CM is competitive on all three domains, while being fast (recall that tests use only one CPU, and no GPU).
Moreover, LTS+CM could also be extended with a value function,
either to PHS*+CM, or by using the value function as input features to each context,
or by binarizing the values into multiple (possibly overlapping) contexts.

For the STP, while the test set used for DeepCubeA is different from the one used 
by the other algorithms,
\footnote{The training set is also different, and although the number of problems is not clearly specified, it appears to be much larger than 50k problems.}
we still expect the results to be comparable since \citet{orseau2021policy}'s test sets are composed of random instances rather than scrambled from the solution.
The heavy cost in expansions paid by DeepCubeA shows the price of finding near-optimal solutions --- which it is reported to find 97\% of the time.
WA* appears to give the best compromise on this problem.
LTS+CM expands more nodes but, given that it is still fast in milliseconds,
we can hope for better results possibly by adding more mutex sets.

DeepCubeA has also been trained on the 900k unfiltered Boxoban set~\cite{boxobanlevels},
and finds short solutions in few expansions.
But because it is trained differently from the other algorithms,
its results are not directly comparable.
We trained LTS+CM on the 450k medium Boxoban set and obtained a slightly better average number of expansions, with a much faster algorithm in milliseconds --- at the expense of solution length.
The same LTS+CM also expands almost 4x fewer nodes on the Boxoban hard set
than the one trained with only 50k problems.

On the Rubik's cube, we report results for LTS+CM at various stages of its training.
After just 300k cubes (1 hour 50 minutes),
scrambled at most 15 times (these are much easier than fully scrambled cubes)
it already finds a policy that solves the whole random test set --- scrambled 100 times, which is usually considered more than sufficient for generating random cubes~\cite{korf1997cube}.
At 400k cubes, the policy is already substantially faster than previous work.
We trained the network for up to 9M cubes, but the average number of expansions stabilizes at around 320.
The training curve is shown in \cref{fig:cube}.
\begin{figure*}[htbp!]
    \centering
    \includegraphics[width=\textwidth]{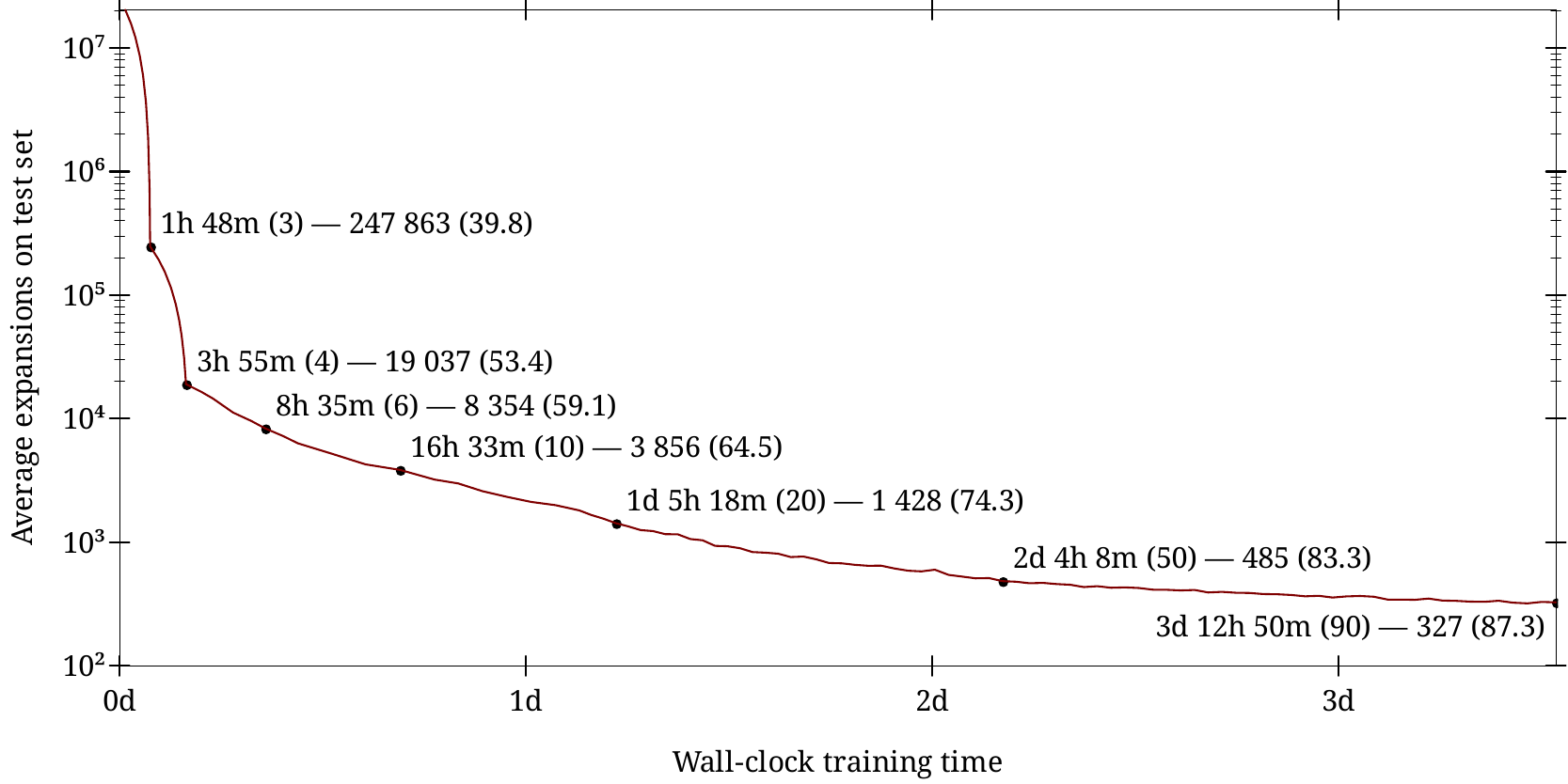}
    \caption{Rubik's cube: Average number of expansions on the test set as a function of training time.
    Each point label should be read: ``training time (iteration) --- expansions (length)'',
    where ``training time'' is the wall-clock time used for both solving and optimizing since the start,
    ``iteration'' is the training set iteration (each set contains 100k cubes),
    ``expansions'' is the average number of expansions on the test set,
    and ``length'' is the average length of the solutions found on the test set.}
    \label{fig:cube}
\end{figure*}

We also compared LTS+CM on the same 100 hardest Rubik's Cube problems of \citet{buchner2022a}
as used by \cite{allen2021focused_macros}.
This set appears slightly simpler than the test set we used as the algorithms are able to find shorter solutions with fewer expansions in this set than in our set of problems. 
Note also that \citet{allen2021focused_macros} do not account for the length of the macro actions in the number of expansions, because they use a  logical representation of the problem to `compress' the results of macro actions --- this assumes access to more information about the environment than just a simulator.
The two approaches, focused macro-actions and learning a policy,
are very much composable, and it would be interesting 
to see whether such macro actions could help make LTS+CM more efficient in training time or converge to a faster policy.
As far as we are aware, LTS+CM is the first machine-learning algorithm to learn a fast policy for the Rubik's cube --- while using only little domain-specific knowledge.

Finally, it must be noted that the timings reported in the table should be read with a grain of salt, as different algorithms have been tested on different machines
and implemented using different libraries. 
LTS+CM has been implemented in Racket~\footnote{https://racket-lang.org/} and we report running times as a means of showing that the CM models are very fast in practice.

\fullresultstrue
\resultstable

\clearpage
\section{Table of Notation}\label{apdx:table_notation}
\begin{tabular}{ll}
$\nodeset$          & Set of all nodes, may contain several root nodes \\
$n$                 & A node in $\nodeset$\\
$d(n)$              & Depth of the node $n$ \\
$\children(n)$      & Children of $n$ \\
$\parent(n)$        & Single parent of $n$, may not exist \\
$\rootop(n)$        & Topmost ancestor of $n$, has no parent \\
$\anc(n)$           & Set of ancestors of $n$ \\
$\ancn(n)$          & $\anc(n)\cup\{n\}$ \\
$\desc(n)$          & Descendants of $n$ \\
$\descn(n)$         & $\desc(n)\cup\{n\}$ \\
$n_{[j]}$           & Node at depth $j$ on the path from  \\
                    & $\rootop(n)=n_{[0]}$ to $n=n_{[d(n)]}$\\
$\nodesetcost(n)$       & Set of nodes of cost $d(\cdot)/\pol(\cdot)$ at most that of $n$\\
$\nodeset_t$        & Set of nodes after $t$ problems \\ 
$\nodeset^*$        & Set of solution nodes $n^*$ \\
$\nodeset^*(n)$     & $=\nodeset^*\cap\descn(n)$ et of solution nodes below $n$ \\ 
$\nodeset^0$        & Set of root nodes (problems) \\
$\leafset(n)$       & Leaves of the tree $\mathcal{N}(n)$ \\
$\pol$              & Policy \\
$\pol(n)$           & Probability of the node $n$ according to \\
                    & the policy $\pol$ \\
$\pol(n' \mid n)$   & $\pol(n') / \pol(n)$, assuming $n'\in\children(n)$. \\
$\dop(n)$           & $d(n)/\pol(n)$ \\
$\loss(n,\beta)$    & Loss function for a single node $=\dop(n)$ for \\
                    & parameters $\beta$ \\
$\Loss(\nodeset', \beta)$ & Cumulative loss over a set of nodes and \\
                    & parameters $\beta$ \\
$\contextset$       & Set of contexts \\
$\contextset(n)$    & Set of contexts active at node $n$ \\
$\patternset$       & Set of \mutexsets{} \\
$M$                 & A \mutexset{}\\
$p_c(a)$            & Probability of the action label $a$ according to a \\
                    &context predictor $p_c$ \\
$\prodmix(n, a)$    & Product mixing of context predictors at node $n$ \\
                    & for action (edge label) $a$  \\
$\beta$             & Parameters of the context predictors \\
$\betaset$          & $= [\epslow, 0]^{|\contextset|\times A}$, set of all possible parameter \\
                    & values for $\beta$ \\
$B_t$               & Budget used at Bootstrap iteration $t$ \\
$\actionset$        & Set of actions (edge labels) \\
$\actionset(n)$     & Set of edge labels at node $n$, possible actions at $n$ \\
$a$                 & An action, edge label \\
$a(n)$              & Edge label (action) from $\parent(n)$ to $n$ \\
$\regret(\nodeset_t)$ & Regret of the learner compared to the optimal \\
                    & parameters for a set of solution nodes $\nodeset_t$ \\
$\indicator{test}$  &=1 if $test$ is true, 0 otherwise
\end{tabular}

\todo{Check for obsolete notation}

\fi

\end{document}